\documentclass[conference]{IEEEtran}
\IEEEoverridecommandlockouts
\usepackage{cite}
\usepackage{amsmath,amssymb,amsfonts}
\usepackage{hyperref}
\usepackage{verbatim}
\usepackage[ruled, vlined, linesnumbered]{algorithm2e}
\usepackage{algpseudocode}

\usepackage{ctable} 
\usepackage{multirow} 

\usepackage{lipsum}
\usepackage{setspace}
\makeatletter

\usepackage[left=0.75in,
            right=0.75in,
            top=0.75in,
            bottom=0.75in]{geometry}

\usepackage{xcolor}
\usepackage{amsthm}
\usepackage{graphicx}
\usepackage{textcomp}
\usepackage{xcolor}
\def\BibTeX{{\rm B\kern-.05em{\sc i\kern-.025em b}\kern-.08em
    T\kern-.1667em\lower.7ex\hbox{E}\kern-.125emX}}

\newtheorem{theorem}{Theorem}

\newtheorem{lemma}{Lemma}

\usepackage{subfig}  
\usepackage[T1]{fontenc}

\usepackage{float}

\usepackage{subcaption}

\begin{document}
\title{\vspace{0.21in} Selective Densification for Rapid Motion Planning in High Dimensions with Narrow Passages\\}

\author{
\IEEEauthorblockN{Lu Huang$^{1}$, \textit{Graduate Student Member, IEEE}, Lingxiao Meng$^{2}$, Jiankun Wang$^{2}$, \textit{Senior Member, IEEE},}
\IEEEauthorblockN{Xingjian Jing$^{1}$, \textit{Senior Member, IEEE}}
\thanks{$^{1}$Lu Huang and Xingjian Jing are with the Department of Mechanical Engineering, City University of Hongkong, Tat Chee Avenue, Kowloon, Hong Kong SAR.
{\tt\small (e-mail: \{lhuang98-c@my., xingjing@\}cityu.edu.hk)}}
\thanks{$^{2}$Lingxiao Meng and Jiankun Wang are with the Department of Electronic and Electrical Engineering, Southern University of Science and Technology, Shen Zhen, China.
{\tt\small (e-mail: \{menglx2021@mail.,wangjk@\}sustech.edu.cn)}}
}


\maketitle

\begin{abstract}
Sampling-based algorithms are widely used for motion planning in high-dimensional configuration spaces. 
However, due to low sampling efficiency, their performance often diminishes in complex configuration spaces with narrow corridors. 
Existing approaches address this issue using handcrafted or learned heuristics to guide sampling toward useful regions. 
Unfortunately, these strategies often lack generalizability to various problems or require extensive prior training.
In this paper, we propose a simple yet efficient sampling-based planning framework along with its bidirectional version that overcomes these issues by integrating different levels of planning granularity. 
Our approach probes configuration spaces with uniform random samples at varying resolutions and explores these multi-resolution samples online with a bias towards sparse samples when traveling large free configuration spaces. 
By seamlessly transitioning between sparse and dense samples, our approach can navigate complex configuration spaces while maintaining planning speed and completeness.
The simulation results demonstrate that our approach outperforms several state-of-the-art sampling-based planners in $\mathbb{SE}(2)$, $\mathbb{SE}(3)$, and $\mathbb{R}^{14}$ with challenging terrains.
Furthermore, experiments conducted with the Franka Emika Panda robot operating in a constrained workspace provide additional evidence of the superiority of the proposed method.
\end{abstract}

\def\abstractname{Note to Practitioners}
\begin{abstract}
Motion planning for robots with a high degree of freedom (DoF), such as humanoid robots and robotic arms, presents significant challenges due to the vast configuration space involved. Each joint adds a dimension, leading to an exponentially larger and more complex configuration space that is difficult to navigate.
To address these challenges, sampling-based techniques are commonly employed. These methods sample points in the configuration space to construct a graph or tree of potential motion paths. However, their computational complexity increases dramatically with environmental complexity. 
Obstacles can create narrow passages in the configuration space, requiring a large number of samples to ensure the connectivity of the underlying graph or tree and, therefore, the existence of solutions.
This paper proposes a novel sampling-based planner along with its bidirectional version that reduces search complexity by simultaneously exploring samples at various resolutions. 
The proposed planners efficiently find solutions in large free configuration spaces by utilizing sparse random samples. 
When navigating narrow passages, they transition to denser random samples to locate viable paths.
By integrating samples of different resolutions, the proposed planners effectively navigate high-dimensional configuration spaces with narrow passages. 
Simulation results demonstrate that our approach outperforms several state-of-the-art sampling-based planners across various scenarios.
The proposed planners, with their adaptability, can be easily implemented as a plug-and-play solution for diverse robotic systems, enabling them to perform autonomous navigation tasks in unstructured environments. 
\end{abstract}

\begin{IEEEkeywords}
Robotics, motion planning, path planning, sampling-based motion planning, high-dimensional motion planning, manipulator motion planning
\end{IEEEkeywords}

\section{Introduction}

\begin{figure}
\centering
  \subfloat[]{\includegraphics[width = 0.2\textwidth]{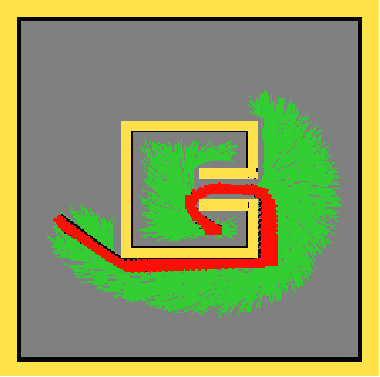}} \hspace{0.03\textwidth}
  \subfloat[]{\includegraphics[width = 0.2\textwidth]{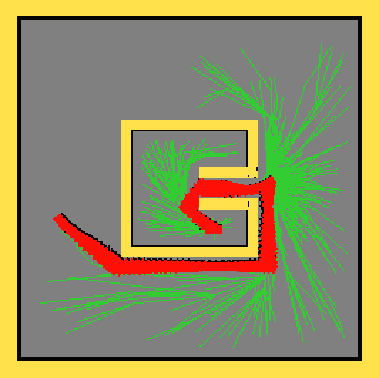}} 

  \subfloat[]{\includegraphics[width = 0.2\textwidth]{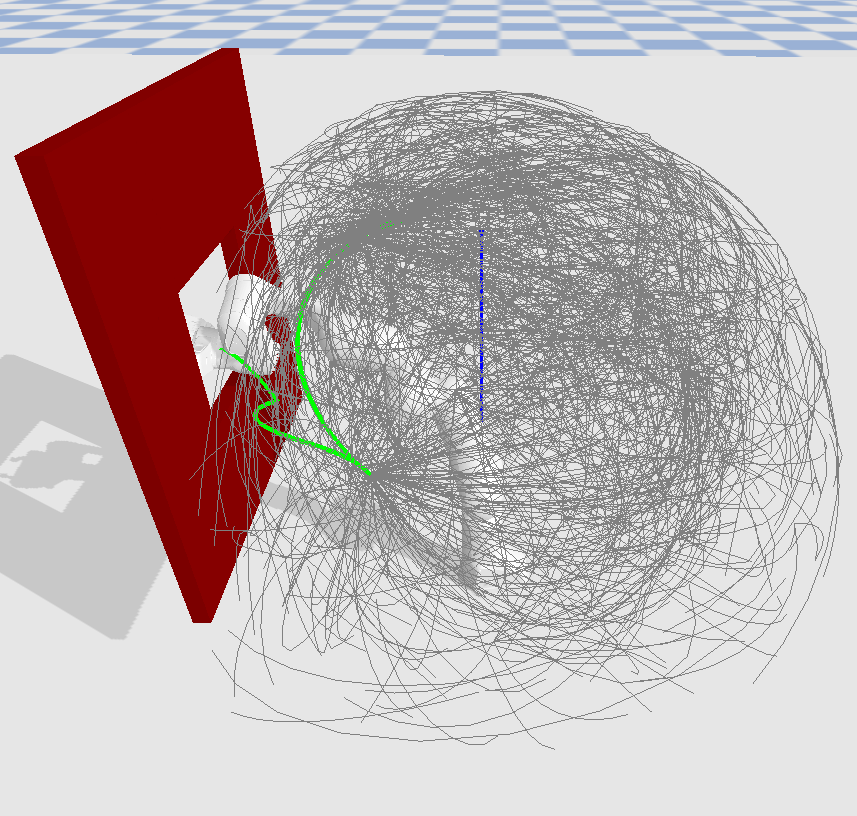}} \hspace{0.03\textwidth}
  \subfloat[]{\includegraphics[width = 0.2\textwidth]{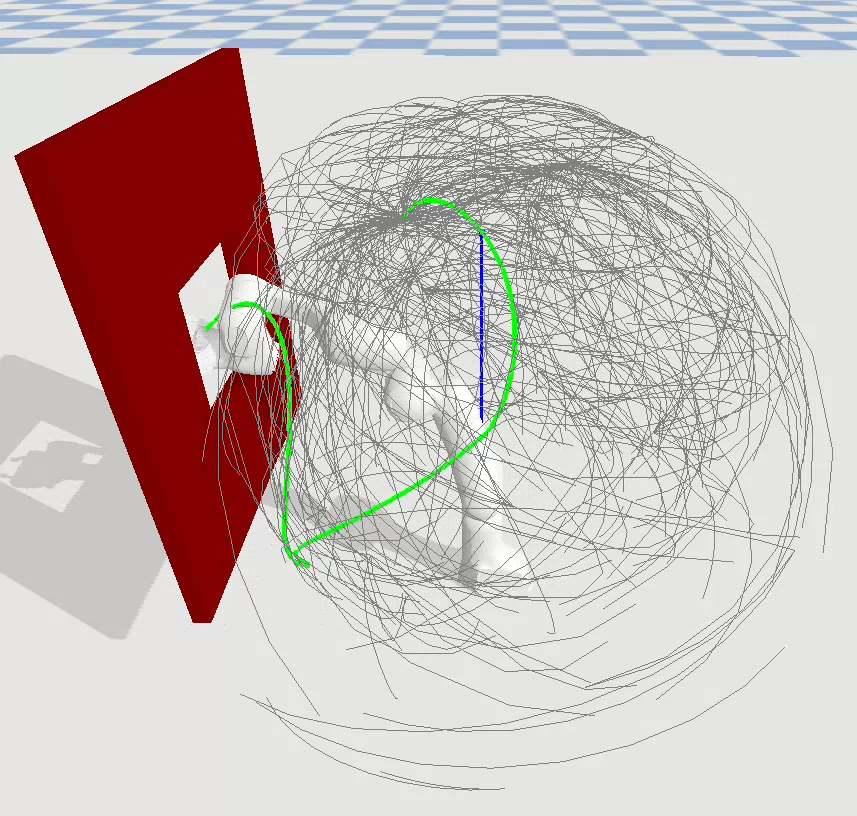}}

\caption{The search results of (a)(c) FMT$^*$ and (b)(d) MRFMT$^*$ in the Bug Trap problem and a 7-DoF manipulation problem are presented. 
In the 7-DoF manipulation problem, the Franka Emika Panda robot is requested to move its end effector across the small hole in the wall.
Both planners probe the free state space with the same number of uniform random samples. 
However, MRFMT$^*$ selectively densifies the search tree near narrow regions of the state space, enabling it to find a solution with a significantly sparser search tree.}
\label{fig:Visualization}
\end{figure}


Sampling-based algorithms, such as Probabilistic Roadmap (PRM) \cite{PRM}, Expansive Space Trees (EST) \cite{EST}, and Rapidly-Exploring Random Tree (RRT) \cite{RRT}, which probabilistically probe state spaces using uniform random samples, have proven to be highly effective in addressing motion planning tasks for high DoF robots under kinematic or dynamic constraints.
RRT$^*$ \cite{RRTstar} and its variants \cite{RRTsharp, RRTX} have shown enhancements in solution completeness and optimality. 
Fast Marching Tree (FMT$^*$) \cite{FMT} further improves efficiency by decoupling the search from the order of random samples. 
It employs forward dynamic programming recursion on a batch of random samples in order of heuristically estimated solution cost, resulting in more efficient exploration of problem domains.
However, for robotic systems operating within a restricted subset of the configuration space, influenced by factors such as environmental layout (e.g., initial and goal configurations, tight corridors), dynamic constraints (e.g., torque limitations), or implicit restrictions (e.g., loop closures), a large number of random samples is required to ensure the connectivity of the underlying search graph, as the probability of placing a random sample inside narrow passageways is minimal, resulting in extensive search cost and degraded performance of sampling-based planners. 

Some approaches strike a balance between completeness and efficiency by biasing samples towards narrow passageways or regions that could improve search efficiency.
Some perform geometric analysis on workspaces or configuration spaces to identify narrow passageways, which can be computationally expensive, especially for robots with complex dynamic or implicit constraints \cite{MediaAxis, WorkspaceImportance, DDRRT, VolumnRRT, Entropy, Utility, InformedRRT}. 
Others avoid explicit geometric description by learning implicit representation of configuration spaces through past planning instances \cite{WeightingFeature, HybridSampling, Bayesian, Learning, LEGO, MDP, CVAE, LocalPremitive, SPARKandFLAME}. 
However, these strategies often encounter limitations in generalizing to diverse problems or require laborious offline training to ascertain the model parameters.

Recently, multi-resolution planning has been proposed \cite{MRA, SelectiveDensification, MultilevelSparseRoadmap}, which approximates configuration spaces using roadmaps with varying densities. 
This method enables fast performance by primarily exploring sparser roadmaps and only densifying them when necessary (e.g., navigating through narrow state spaces). 
Although the existing approaches bypass the need for prior environmental knowledge or offline training, they still require advanced discretization of configuration spaces or well-crafted heuristic functions to guide exploration effectively, making them impractical for online motion planning in complex problem domains.

We aim to develop a general, efficient approach for fast online motion planning in high-dimensional restricted configuration spaces. 
To this end, we propose the Multi-Resolution FMT$^*$ (MRFMT$^*$) algorithm, which performs biased exploration of configuration spaces online without a prior discretization of configuration spaces or problem-specific heuristics by combining the advantages of random samples with different resolutions.
Specifically, MRFMT$^*$ applies multiple FMT$^*$ searches running on uniformly random samples with different resolutions simultaneously.
It searches for solution paths mainly over sparse random samples and shifts to denser samples only in areas where the underlying search graph of sparse samples does not contain a feasible path.
This selective densification scheme ensures that the planner efficiently navigates the broader landscape and spends most computational efforts in challenging regions.
Moreover, we introduce a bidirectional version of this approach, Bidirectional MRFMT$^*$ (BMRFMT$^*$), to further improve planning speed in high-dimensional state spaces. 
We visualize the search results in the Bug Trap problem in Fig.\ref{fig:Visualization}.

This paper is structured as follows. 
In section \ref{Section:RelatedWorks}, we review related works.
We present background introductions and detailed descriptions for the proposed methods in Section \ref{Section:Algorithm}. 
We prove the completeness and asymptotic optimality of MRFMT$^*$ in Section \ref{Section:Analysis}.
Section \ref{Section:Simulation} presents simulation results comparing the proposed method with various state-of-the-art planners to verify the advantages of our method numerically. 
In Section \ref{Section:Experiment}, we further illustrate the superiority of the proposed method in practical applications by applying it to the manipulation task of the Franka Emika Panda robot operating in a highly constrained workspace.
Section \ref{Section:Discussions} discusses the extension of the proposed algorithms to planning under differential constraints and examines the effects of the algorithm parameters on performance.
At the end of this paper, we discuss the results and make conclusions.

\section{Related Works}
\label{Section:RelatedWorks}
\subsection{Biased Sampling}
Over the past few decades, biased sampling strategies have been the subject of extensive research. 
Several geometric methods focus on analyzing the medial axis \cite{MediaAxis} or performing tetrahedralization \cite{WorkspaceImportance} of the workspace. 
These techniques help identify constrained areas and construct probability distributions to sample them preferentially \cite{DynamicRegionRRT}.
In \cite{Entropy, Utility}, the informativeness of new samples is assessed using geometry-based utility functions designed to maximize roadmap coverage and connectivity. 
Some approaches employ biased sampling based on the structure of the search tree. 
For instance, Dynamic-Domain RRT biases samples towards obstacle boundaries and tree nodes \cite{DDRRT}, while the Ball Tree algorithm approximates the free state space using size-varying balls around the search tree \cite{VolumnRRT}.
Instead of relying on uniform sampling to ensure solution completeness, Informed RRT$^*$ \cite{InformedRRT} takes a different approach by deterministic sampling from a prolate hyperspheroid defined by the current solution cost, thus encompassing all potential samples that could improve the solution.

Geometric analysis of high-dimensional configuration spaces can be computationally intensive, particularly for robotic systems with complex dynamic constraints or implicit limitations. 
To address this challenge, various learning-based approaches have been explored.
Zucker et al. \cite{WeightingFeature} propose a method for learning a locally optimal weighting of workspace features, thereby adjusting the significance of different features when sampling. 
Hsu et al. introduce a hybrid sampling strategy combining multiple samplers, each selectively activated based on probabilities learned from experience \cite{HybridSampling}. Similarly, Lai et al. frame the sequential sampling problem as a Markov process and apply sequential Bayesian updating to refine the local proposal distribution of sampling points based on past observations \cite{Bayesian}.
Rather than relying on past experience, in \cite{Vonasek}, the extension of the search tree is guided by a solution of a relaxed version of the original problem.

Recent advancements in deep learning have led to efforts focused on learning latent representations of configuration spaces and inferring efficient sampling distributions using neural networks. 
In particular, the robot state and environment are mapped to a distribution over paths, as demonstrated in \cite{Learning, NeuralRRT}. 
Ichter et al. \cite{CVAE} train a Conditional Variational AutoEncoder (CVAE) to transform a uniform Gaussian distribution into a biased sampling distribution that favors the predicted shortest path based on the environmental layout.
The LEGO framework \cite{LEGO} alleviates the learning burden by training a CVAE to generate bottleneck samples through which a near-optimal path must pass. 
Additionally, in \cite{LocalPremitive, SPARKandFLAME}, a batch of local samplers is trained to capture workspace features that create "challenging regions" within the configuration space. These local samplers are combined to generate a biased global sampler, enabling effective planning in workspaces of arbitrary sizes.

Some research approaches view biased sampling as a decision-making process and employ reinforcement learning methods to learn the optimal policy online. 
For instance, in \cite{MDP}, the optimal sampling distribution is converted into a sample rejection procedure and framed as a Markov Decision Process (MDP), which is modeled using a fully connected neural network.
RRF$^*$ \cite{RRF} expands multiple local search trees simultaneously from different roots and learns the optimal tree expansion sequences through a Multi-Armed Bandit (MAB) formulation. 
Faroni and Berenson enhance kinodynamic RRT by applying the MAB algorithm to learn the optimal transition sequence \cite{MAB_kinoRRT}.

While learning-based techniques show promise in integrating environmental and system knowledge to enhance sampling-based motion planning, their substantial offline training requirements may limit their practicality for tasks that demand a plug-and-play planner. Moreover, many of these approaches necessitate high-performance computing platforms, which can significantly increase overall system costs.

\subsection{Multi-resolution Planning}
Multi-resolution planning was initially adopted by some graph-based approaches \cite{Likhachev, Smooth}, which discretizes the configuration space to a multi-resolution lattice tailored to specific environments and problems in prior in order to accelerate online planning speed.  
The Multilevel Sparse Roadmap algorithm\cite{MultilevelSparseRoadmap} adapts multi-resolution planning for online use by combining roadmaps at various resolutions. 
It precomputes a sequence of roadmaps with an increasing number of evenly distributed edges and iteratively searches for the shortest path from the sparsest to the densest roadmap. 
This approach enables a robot to quickly find an initial path, which can then be refined to yield a more optimal solution. 
Nonetheless, a significant drawback is that the search expense can be particularly high if no path exists in the sparse roadmaps, even though the dense roadmaps are only crucial in limited regions.

Search efficiency can be improved by making the sparse and dense roadmaps complementary. 
Several approaches selectively densify specific regions while primarily searching over sparse roadmaps \cite{MRA, SelectiveDensification}. 
The MRA$^*$ algorithm \cite{MRA} employs multiple searches across different roadmaps. 
When the search on coarser roadmaps is inadmissible, MRA$^*$ switches to the finest resolution for further exploration. 
However, this reliance on the finest resolution can dominate the planning process, even when it can not significantly enhance the solution quality.
In contrast, the Selective Densification algorithm \cite{SelectiveDensification} connects roadmaps using zero-cost cross-layer edges and searches the combined graph with the Weighted A$^*$ algorithm \cite{WA}. 
The heuristics for each layer are scaled by an inflation factor proportional to the layer's density, which biases the search toward sparser layers unless the solution path requires traversing denser roadmap samples.
Nonetheless, the inflation parameters must be carefully chosen beforehand to ensure optimal performance, as a too-small parameter leads to costly searches over the entire layered graph, while a too-large one results in a search being overly dominated by heuristics.

As a multi-resolution motion planner, our approach does not require offline training or prior knowledge of the environment.
Distinguished from the existing schemes, ours avoids the need for in prior configuration space discretization or roadmap construction.
Instead, it probes the configuration space probabilistically by random samples and constructs a search graph over the random samples online. 
Furthermore, our approach eliminates the need for manually tuned parameters to guide the search across different resolutions, enhancing its practicality for real-world applications.
Ours also adopts a lazy collision check scheme when expanding the search graph, which significantly reduces the number of edge evaluations during the whole planning procedure, cutting down the overall search time for a feasible solution compared with the previous approaches.

\section{Multi-resolution Sampling-based Planners}
\label{Section:Algorithm}

Let $\mathcal{X}$ be a $d$-dimensional configuration space ($d\geq 2$) with free space $\mathcal{X}_{free}$ and configurationspace obstacles $\mathcal{X}_{obs} = cl(\mathcal{X}\backslash \mathcal{X}_{free})$. 
A path planning problem can then be characterized by a triplet $(\mathcal{X}_{free}, x_{init}, \mathcal{X}_{goal})$.
A feasible solution path $\pi$ for the planning problem $(\mathcal{X}_{free}, x_{init}, \mathcal{X}_{goal})$ is an ordered set of collision-free configurations from the initial condition $x_{init}\in\mathcal{X}_{free}$ to the goal region $\mathcal{X}_{goal}\subset\mathcal{X}_{free}$, i.e., $\pi(0) = x_{init}$ and $\pi(1)\in \mathcal{X}_{goal}$.
Denote the set of all paths by $\Sigma$.
A cost function for the planning problem is $c: \Sigma\rightarrow \mathbb{R}_{\geq 0}$. 
Let $\Pi$ be the set of all feasible solutions. The optimal solution $\pi^*$ is the one with $c(\pi^*) = min_{\pi\in \Pi}c(\pi)$. 
Our goal is to find a feasible solution path with a cost close to $\pi^*$.
The motion planner should report failure to the user if no feasible path exists.



\subsection{High-level description}
MRFMT$^*$ stands out in its approach to approximating configuration spaces. It does so at the beginning of planning using uniform random samples generated at multiple resolutions, a unique method that distinguishes it from traditional sampling-based planners. 
Traditional sampling-based planners either employ an incremental sampler that generates one sample \cite{RRT, PRM, RRTstar} or a batch of samples \cite{BIT, Densification} at a time during planning to refine their approximation of the configuration space. 
Alternatively, some planners approximate the configuration space in advance using a predetermined number of random samples \cite{FMT}.
One of the key challenges faced by traditional planners is the significant increase in average neighboring size and overall search cost as the total number of samples in the graph increases. 
This is particularly problematic for problems where a dense underlying graph is necessary for covering narrow passages.
MRFMT$^*$ addresses this challenge by conducting a multi-resolution search, simultaneously exploring samples at different resolutions. 
It prioritizes searching the coarsest samples during the planning procedure, transitioning to finer resolution samples only when the coarsest samples fail to produce a feasible solution.
The concept of prioritizing coarsest samples to mitigate the complexity of high-resolution graph searches was also adopted by EIRM$^*$ \cite{EIRM}, a multi-query sampling-based planner that resets to the coarsest samples at the start of each new query. However, EIRM$^*$ still relies on incrementally densifying the graph to find feasible solutions during queries.

MRFMT$^*$ performs local planning from the initial sample to iteratively identify a solution path. 
In particular, it applies the forward dynamic programming recursion proposed by FMT$^*$ \cite{FMT} across the multi-resolution samples to generate a tree of paths that gradually expand in cost-to-come space. 
This recursion expands the sample with the lowest estimated solution cost and selectively establishes locally optimal connections for its neighboring samples.
If a locally optimal connection (assuming no obstacles) intersects an obstacle, that neighboring sample is skipped and revisited later instead of seeking alternative locally optimal connections in the vicinity like RRT$^*$.
Since the forward dynamic programming recursion evaluates only the edges with locally optimal solution costs, the overall number of edge evaluations of FMT$^*$ is significantly reduced while ensuring asymptotically optimal solutions.
The difference from FMT$^*$ is that MRFMT$^*$ constrains expansion to the samples with the same resolution, which further reduces the number of neighboring queries and edge evaluations.
The forward dynamic programming recursion enables MRFMT$^*$ to simultaneously conduct fast graph construction and solution search, in contrast to previous multi-resolution planners \cite{MultilevelSparseRoadmap, MRA, SelectiveDensification} that must either separate these phases or perform numerous expensive edge evaluations.

Bidirectional search is a widely used technique to enhance the performance of motion planners and has demonstrated its efficiency in various applications \cite{RRTconnect, BFMT, SelectiveDensification}. 
The idea of bidirectional search is to simultaneously propagate two search wavefronts initiated from the start and the goal samples, aiming to meet in between. 
BMRFMT$^*$ combines the bidirectional search strategy with the proposed multi-resolution sampling-based planner to enhance planning efficiency in high-dimensional configuration spaces.

\subsection{Multi-resolution Random Samples}
\label{RandomSampleSetDef}
Consider a sequence of $N$ unique configurations uniformly sampled from $\mathcal{X}$, denoted as $(x_1, x_2, \cdots, x_N)$, along with a strictly increasing sequence of $L$ positive constants, $(0<n_1 < n_2 < \cdots <n_L = N)$, where $L<<N$. 
The random sample set with the $l^{th}$ dense sparsity level, denoted as $X_l$, consists of the samples representing the first $n_l$ configurations. 
Additionally, we include the goal and initial configuration samples $x_{goal}$ and $x_{init}$ in every sample set.
For each sample set $X_l$, we calculate the sample connection radius $r_l$ or the number of neighbors $k_l$ as a function of the sample set size $n_l$ by following the same principles as Rapidly-Exploring Random Graph (RRG) or kPRM \cite{RRTstar}.
Let $x_i^l$ represent a sample from $X_l$. 
The counterpart samples $x_i^{l-1}$ and $x_i^{l+1}$ correspond to samples with the same configuration but in adjacent resolutions. 
Note that $x_i^{l-1}$ does not always exist as $n_{l-1} < n_l$.
We define zero-cost edges between these counterparts, allowing MRFMT$^*$ to traverse among samples with different resolutions. 
The neighbors of a sample $x_i^l$, therefore, contain
(i) the samples within the same resolution that are either within a distance $r_l$ from $x_i^l$ in an r-disk underlying search graph or among the $k_l$ nearest neighbors in a k-nearest underlying search graph; 
(ii) the counterpart samples $x_i^{l-1}$  (for $l > 1$) and $x_i^{l+1}$ (for $l < L$).

\subsection{MRFMT$^*$ - Detailed Description}
\label{MRFMT_detail}

\begin{figure}
    \centering
    \subfloat[]{\includegraphics[width=0.15\textwidth]{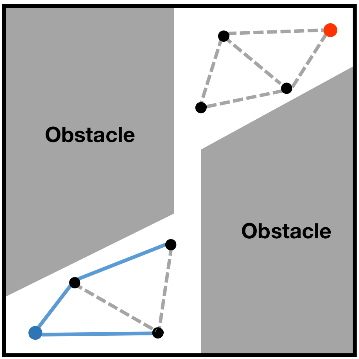}}\hfill
    \subfloat[]{\includegraphics[width=0.15\textwidth]{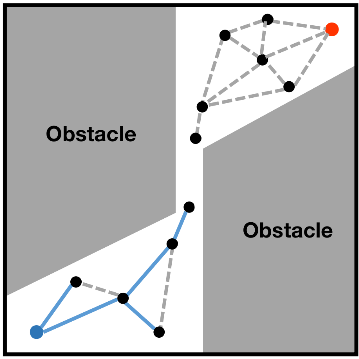}}\hfill
    \subfloat[]{\includegraphics[width=0.15\textwidth]{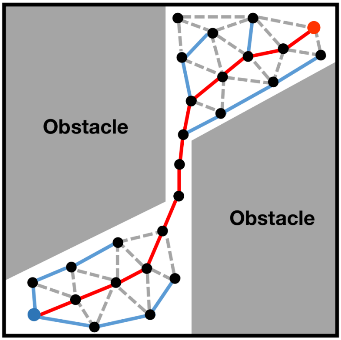}}
    \caption{The search results of FMT$^*$ over (a) sparse, (b) medium, and (c) dense approximations of a 2D environment with a narrow corridor. Over the sparse and medium approximations, FMT$^*$ fails to find a solution due to the disconnectedness of the underlying search graphs. }
    \label{fig:2D_narrow_corridor}
\end{figure}

\begin{algorithm}[htbp]
\setstretch{1.0}
    \caption{MRFMT$^*$}
    \label{MRFMT}
    \SetKwFunction{Initialize}{Initialize}
    \SetKwFunction{SampleFree}{SampleFree}
    \SetKwFunction{Neighbor}{Neighbor}
    \SetKwFunction{Expand}{Expand}
    \SetKwFunction{Main}{Main}
    \SetKwFunction{Path}{Path}
    \SetKwFunction{Cost}{Cost}
    \SetKwFunction{ChooseExpansionNode}{ChooseExpansionNode}
    \SetKwFunction{PTC}{PTC}
    \SetKwFunction{Swap}{Swap}
    \SetKwFunction{AddNode}{AddNode}

    \SetKwProg{Fn}{Function}{:}{}
    \newcommand{\algrule}[1][.2pt]{\par\vskip.1\baselineskip\hrule height #1\par\vskip.1\baselineskip}
    
    \textbf{Input:} $(\mathcal{X}_{free}, x_{init}, \mathcal{X}_{goal})$, $c(\cdot)$, $L$, $\mathbf{X}=\{X_l\}_{l=1,\cdots,L}$\;
    
        $\mathcal{T}\leftarrow$ \Initialize{$\mathbf{X}$, $x_{init}^1$}\; \label{initialize}
        
        $z\leftarrow$ $x_{init}^1$ \Comment{The wavefront node for expansion} \; \label{initialize_z}
        $p\leftarrow 1$ \Comment{The sparsity level pointer} \; \label{initialize_pointer} 
        \While{not \PTC{}}
        {
            \If{not \Expand{$\mathcal{T}$, $z$, $p$}} 
            {
                \label{checkExpansionSuccessful}
                break \; \label{break}
            }
            $z\leftarrow \arg\min_{x\in V_{open, p}}\{\Cost(\mathcal{T}, x)\}$ \; \label{MAIN:updateZ}
            
        }

        \eIf{$z\in \mathcal{X}_{goal}$}
        {
            \KwRet{\Path{$\mathcal{T}$, $z$}}\; \label{Track_Path}
        }
        {
            \KwRet{Failure} \; \label{Report_Failure}
        }

    \algrule
    \Fn{\Initialize{$\mathbf{X}$, $x$}}{
      $V\leftarrow\{x\}$,$E\leftarrow\phi$,
      $V_{unvisited}\leftarrow \mathbf{X}\backslash\{x\} $\; 
      $V_{open,1}\leftarrow\{x\}$, $V_{open,l}\leftarrow\phi$ for $l=2,\cdots, L$ \;}
      \KwRet{($V$, $E$}, $V_{unvisited}$, $\{V_{open,l}\}_{l=1,\cdots, L}$)\;

\end{algorithm}    

\begin{algorithm}[htbp]
\setstretch{1.0}
    \caption{BMRFMT$^*$}
    \label{BMRFMT}
    \SetKwFunction{Initialize}{Initialize}
    \SetKwFunction{SampleFree}{SampleFree}
    \SetKwFunction{Neighbor}{Neighbor}
    \SetKwFunction{Expand}{Expand}
    \SetKwFunction{Main}{Main}
    \SetKwFunction{Path}{Path}
    \SetKwFunction{Cost}{Cost}
    \SetKwFunction{ChooseExpansionNode}{ChooseExpansionNode}
    \SetKwFunction{PTC}{PTC}
    \SetKwFunction{Swap}{Swap}
    \SetKwFunction{AddNode}{AddNode}

    \SetKwProg{Fn}{Function}{:}{}
    \newcommand{\algrule}[1][.2pt]{\par\vskip.1\baselineskip\hrule height #1\par\vskip.1\baselineskip}
    
    \textbf{Input:} $(\mathcal{X}_{free}$, $x_{init}$, $\mathcal{X}_{goal})$, $c(\cdot)$, $L$, $\mathbf{X}=\{X_l\}_{l=1,\cdots,L}$\;
        $\mathcal{T}\leftarrow$ \Initialize{$\mathbf{X}$, $x_{init}^1$}\; \label{BMRFMT:initialize}
        {\color{blue}
        $\mathcal{T}'\leftarrow$ \Initialize{$\mathbf{X}$, $x_{goal}^1$} \; 
        \label{BMRFMT:initialize_}
        }
        
        $z\leftarrow$ $x_{init}^1$ {\color{blue}, $x_{meet}\leftarrow \phi$} \; \label{BMRFMT:initialize_z}
        $p\leftarrow 1$ {\color{blue}, $p'\leftarrow 1$}\; \label{BMRFMT:initialize_pointer}
        \While{not \PTC{}}
        {
            \eIf{not \Expand{$\mathcal{T}$, $z$, $p$ \textcolor{blue}{, $\mathcal{T}'$, $x_{meet}$}}} 
            {
                {\color{blue}
                \label{BMRFMT:checkExpansionSuccessful}
                \eIf{$V_{open,p'}'\neq \phi$} 
                {
                    \label{BMRFMT:SwapTreeIfCurrentFailsBegins}
                    \Swap{$\mathcal{T}, \mathcal{T}'$}\; 
                    \label{BMRFMT:SwapTreeIfCurrentFailsEnds}
                }
                {
                     break \; \label{BMRFMT:break}
                }     
                }
            }
            {\color{blue}
            {
                \If{$V_{open,p'}'\neq \phi$}
                {
                    \label{BMRFMT:SwapIfTheOtherTreeIsNotEmptyBegins}
                    \Swap{$\mathcal{T}, \mathcal{T}'$}\; 
                    \label{BMRFMT:SwapIfTheOtherTreeIsNotEmptyEnds}
                }
            }
            }
            $z\leftarrow \arg\min_{x\in V_{open, p}}\{\Cost(\mathcal{T}, x)\}$ \; \label{BMRFMT:updateZ} 
        }
        \color{blue}{
        \eIf{$x_{meet}\neq \phi$}
        {
            \label{BMRFMT:track_solution_path_begins}
            \KwRet{\Path{$\mathcal{T}$, $x_{meet}$} $\cup$ \Path{$\mathcal{T}'$, $x_{meet}$}}\; 
            \label{BMRFMT:track_solution_path_ends}
        }
        {
            \KwRet{Failure} \; \label{BMRFMT:report_failure}
        }
        }

\end{algorithm}    

\begin{algorithm}[ht]
\setstretch{1.0}
    \caption{Expand($\mathcal{T}, z, p$ \textcolor{blue}{, $\mathcal{T}', x_{meet}$})}
    \label{Expand}
    \SetKwFunction{CollisionFree}{CollisionFree}
    \SetKwFunction{Resolution}{Resolution}
    \SetKwFunction{Neighbor}{Neighbor}
    \SetKwFunction{Expand}{Expand}
    \SetKwFunction{Cost}{Cost}
    \SetKwFunction{AddNode}{AddNode}

    \SetKwProg{Fn}{Function}{:}{}
    \newcommand{\algrule}[1][.2pt]{\par\vskip.1\baselineskip\hrule height #1\par\vskip.1\baselineskip}

    $V_{new}\leftarrow\phi$  \hspace{0.05in} // A node set for saving updated nodes\;
    $Z_{near}\leftarrow$\Neighbor{$z$}$\cap V_{unvisited}$\; \label{unvisited}
    \For{$x\in Z_{near}$}
    {
        $X_{near}\leftarrow Neighbor(x)\cap V_{open, p}$\; \label{find_X_near}
        $x_{min}\leftarrow \arg\min_{x'\in X_{near}}\{$\Cost{$\mathcal{T}, x'$} $+\hat{c}(x',x)\}$\; \label{find_x_min}
        \If{\CollisionFree{$x_{min}, x$}}
        {
            \label{edge_evaluation}
            \AddNode{$V, E, V_{unvisited}, V_{new}, x_{min}, x$}\;\label{addnode}
            {\color{blue} 
                \If{$x\in V' \ \And $ 
                \Cost{$\mathcal{T}, x$} + \Cost{$\mathcal{T}', x$} <
                \Cost{$\mathcal{T}, x_{meet}$} + \Cost{$\mathcal{T}', x_{meet}$}
                }{
                    \label{Expand:UpdateXmeetBegin}
                    $x_{meet}\leftarrow x$ \;
                } \label{Expand:UpdateXmeetEnd}
            }
        }
    }
    $V_{open,p}\leftarrow V_{open,p} \backslash \{z\}$ \; \label{remove_z_from_open}
    
    \For{$v\in V_{new}$}
    {
        \label{Add_update_to_open_begin}
        $l\leftarrow$\Resolution{$v$}\;
    
        $V_{open,l}\leftarrow V_{open,l} \cup \{ v \}$ \; \label{Add_update_to_open}

        \textbf{if} $l<p$ \textbf{then} $p\leftarrow l$ \; \label{Decrease_Pointer}
    }

    \While{$V_{open, p}=\phi$}{ 
        \eIf{$p<L$}
        {
            $p\leftarrow p+1$ \; \label{Increase_Pointer}
        } 
        {
            \KwRet{False} \hspace{0.05in} \Comment{Failed expansion}\; \label{terminate}
        }
    }
    
    \KwRet{True} \hspace{0.05in} \Comment{Successful expansion} \; 
    \algrule
    \Fn{\AddNode{$V, E, V_{unvisited}, V_{open}, x', x$}}{
          $V\leftarrow V\cup\{x\}$, 
          $E\leftarrow E\cup\{(x', x)\}$\;
          $V_{open}\leftarrow V_{open}\cup\{x\}$\;
          $V_{unvisited}\leftarrow V_{unvisited}\backslash\{x\}$\; \label{remove_unvisited}
     }
     
\end{algorithm}    

The main algorithm of MRFMT$^*$ is presented in Algorithm \ref{MRFMT}. Throughout the search, the algorithm maintains a tree $\mathcal{T}$. It stores successful connections in a sample set $V$, storing the samples connected to the tree, and an edge set $E$, storing pairs of tree samples with feasible motions between them.
MRFMT$^*$ associates each sample set with a priority queue $V_{open,l}$ where $l$ is the sparsity level of the sample set, to track the search wavefront of $\mathcal{T}$ in sorted order of estimated solution cost (e.g., cost-to-come).
The set $V_{unvisited}$ contains all unvisited samples.
To simplify notation, $\mathcal{T}$ is referred to as a quaternion $(V, E, V_{unvisited}, \{V_{open, l}\}_{l=1,\cdots, L})$.

Before diving into the algorithm details, we briefly list the functions employed by the algorithm. 
The function \FuncSty{SampleFree($N$)} returns a set of $N$ i.i.d. random samples from $X_{free}$. 
\FuncSty{Resolution($x$)} returns the sparsity level of $x$. 
$\hat{c}(x,y)$ returns the solution cost from $x$ to $y$ in the assumption that the motion between $x$ and $y$ is collision-free.
\FuncSty{Neighbor($x$)} returns the neighboring samples of $x$ as defined in Section \ref{RandomSampleSetDef}.
\FuncSty{PATH($\mathcal{T}, x$)} returns the unique path in $\mathcal{T}$ from its root to sample $x$, and
\FuncSty{Cost($\mathcal{T}, x$)} returns the cost of the path. 
\FuncSty{CollisionFree($x, y$)} is a boolean function returning true if the movement between configurations $x$ and $y$ is collision-free. 
\FuncSty{Swap($\mathcal{T}, \mathcal{T}'$)} swaps the two trees $\mathcal{T}$ and $\mathcal{T}'$. 
\FuncSty{PTC()} represents the problem termination criterion (e.g., the search time runs out, a solution from $x_{init}$ to $\mathcal{X}_{goal}$ is found, etc.).

We are now ready to explain the algorithm in detail. 
MRFMT$^*$ starts by initializing a search tree $\mathcal{T}$ rooted at $x_{init}^1$ using the \FuncSty{Initialize} procedure (line \ref{initialize}, Algorithm \ref{MRFMT}), where $x_{init}^1$ is the sample of $X_1$ which corresponds to the initial configuration. 
\text{\color{black}$x_{init}^1$} is inserted into $V_{open, 1}$ for expansion.
Once initialization is complete, MRFMT$^*$ begins expanding $\mathcal{T}$ using the \FuncSty{Expand} procedure. 
We denote $p$ as the sparsity level over which MRFMT$^*$ is searching for expansion, which is initialized to be 1 (line \ref{initialize_pointer}, Algorithm \ref{MRFMT}).
The \FuncSty{Expand} procedure finds the locally optimal connection for the unvisited neighboring samples of $z$, which is the lowest-cost open sample at the current sparsity level $p$ and is initially set to $x_{init}^1$ for the first search (line \ref{initialize_z}, Algorithm \ref{MRFMT}). 
For each unvisited neighboring sample $x\in Z_{near}$, the \FuncSty{Expand} procedure firstly searches its neighboring samples with sparsity level $p$ and in the open set (line \ref{find_X_near}, Algorithm \ref{Expand}).
The edge between $x$ and the open sample with the minimal heuristically-estimated solution cost, $x_{min}$, is evaluated (line \ref{edge_evaluation}, Algorithm \ref{Expand}). 
If the edge is collision-free, $x$ is added to the search tree and removed from $V_{unvisited}$ (line \ref{addnode}, Algorithm \ref{Expand}); otherwise, $x$ is skipped.
Any updated sample is opened and inserted into its associated priority queues (line \ref{Add_update_to_open_begin}-\ref{Add_update_to_open}, Algorithm \ref{Expand}).
If there is an updated sample with a sparser resolution, 
$p$ is then decreased (line \ref{Decrease_Pointer}, Algorithm \ref{Expand}), and therefore, the next \FuncSty{Expand} procedure will traverse to the sparser resolution. 
If $V_{open, p}$ is empty after an expansion, indicating that there are no open samples for expansion at the current sparsity level, MRFMT$^*$ will turn to the denser resolution by increasing $p$ by 1 (line \ref{Increase_Pointer}, Algorithm \ref{Expand}).
If there is no wavefront node, i.e., $V_{open, l}$ is empty for all $l=1,\cdots,L$, the search will terminate (line \ref{terminate}, Algorithm \ref{Expand}). 
A solution path can be traced back from $z$ if $z \in \mathcal{X}_{goal}$ (line \ref{Track_Path}, Algorithm \ref{MRFMT}).

\begin{figure*}
\centering
  \subfloat[MRFMT$^*$ is initialized with $p=1$. ]{\includegraphics[width = 0.24\textwidth]
  {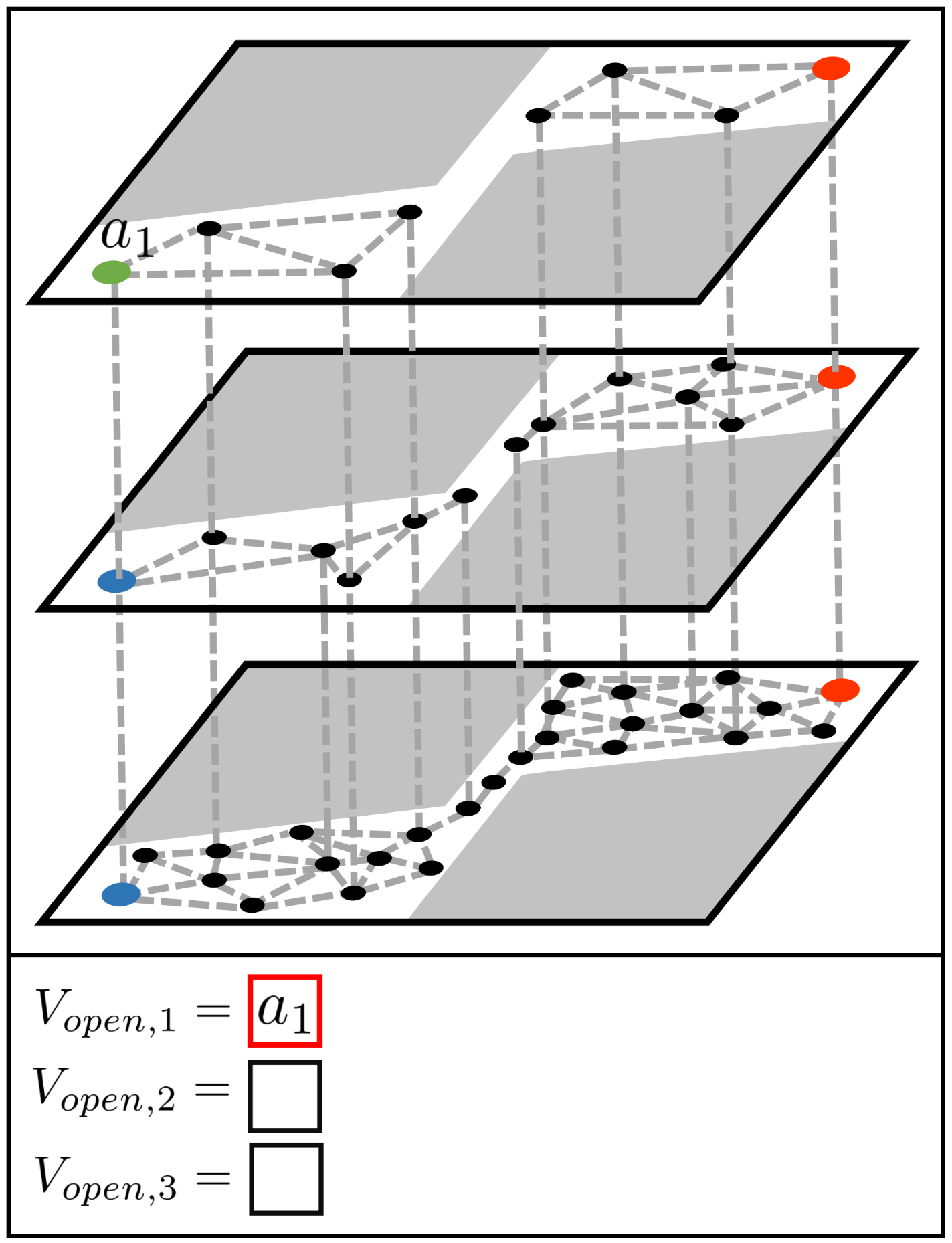}}\hfill
  \subfloat[$a_1$ to $d_1$ were iteratively expanded and closed. $V_{open, 1}$ is empty, and therefore $p$ is increased to 2.]{\includegraphics[width = 0.24\textwidth]
  {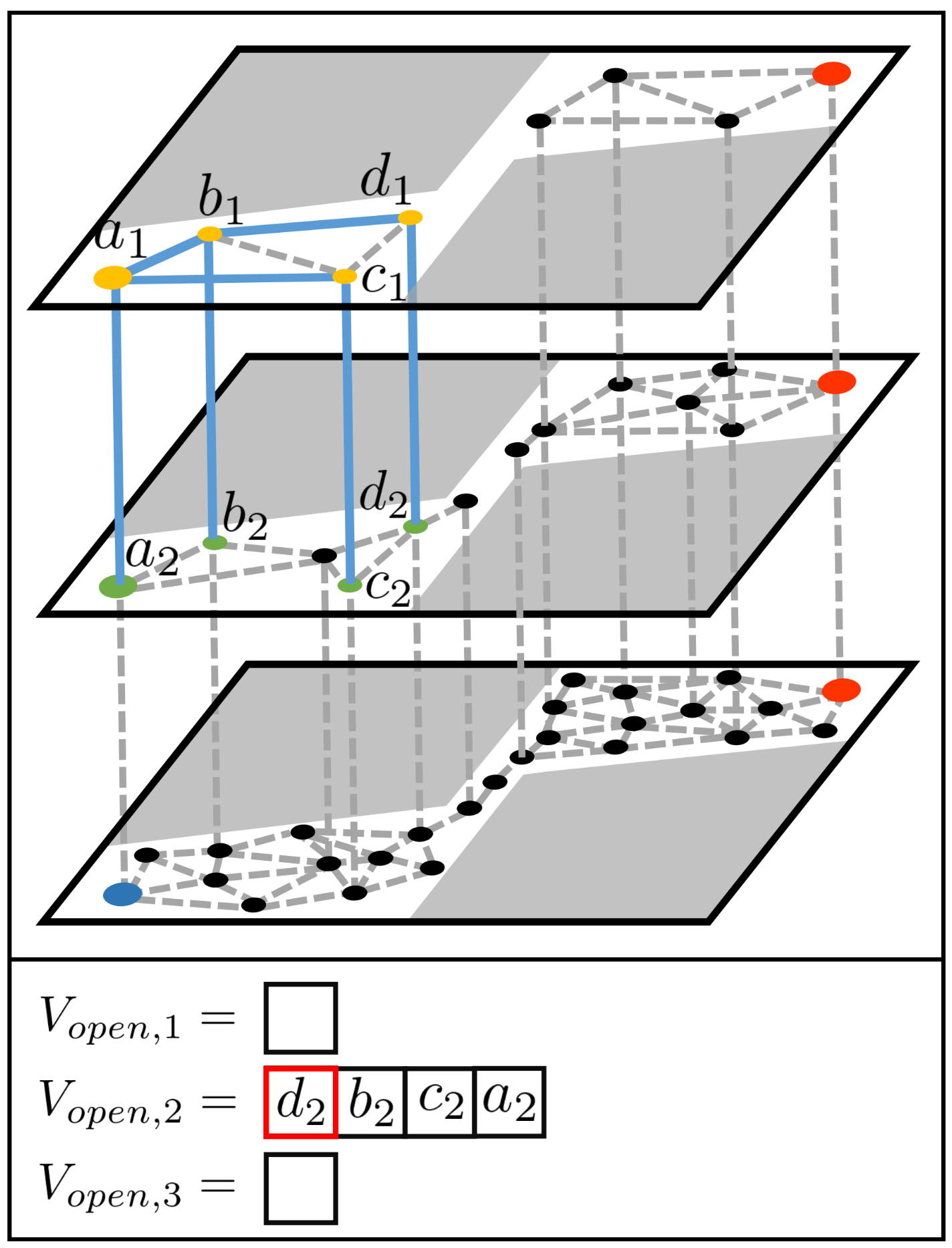}}\hfill
  \subfloat[$d_2$ was expanded and closed.]{\includegraphics[width = 0.24\textwidth]
  {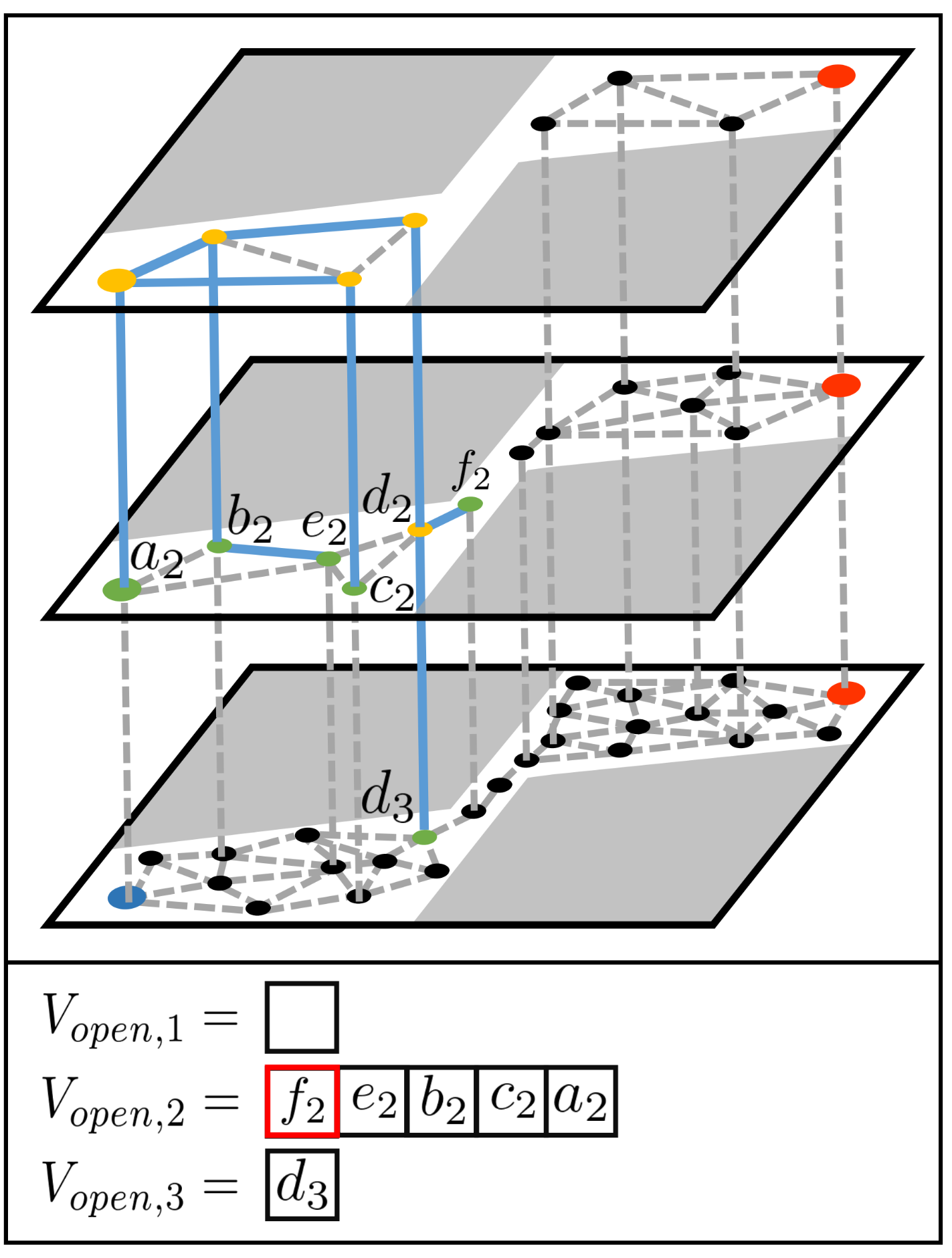}}\hfill
  \subfloat[$a_2$ to $f_2$ were iteratively expanded and closed. $V_{open,2}$ is empty, and therefore $p$ is increased to 3.]{\includegraphics[width = 0.24\textwidth]
  {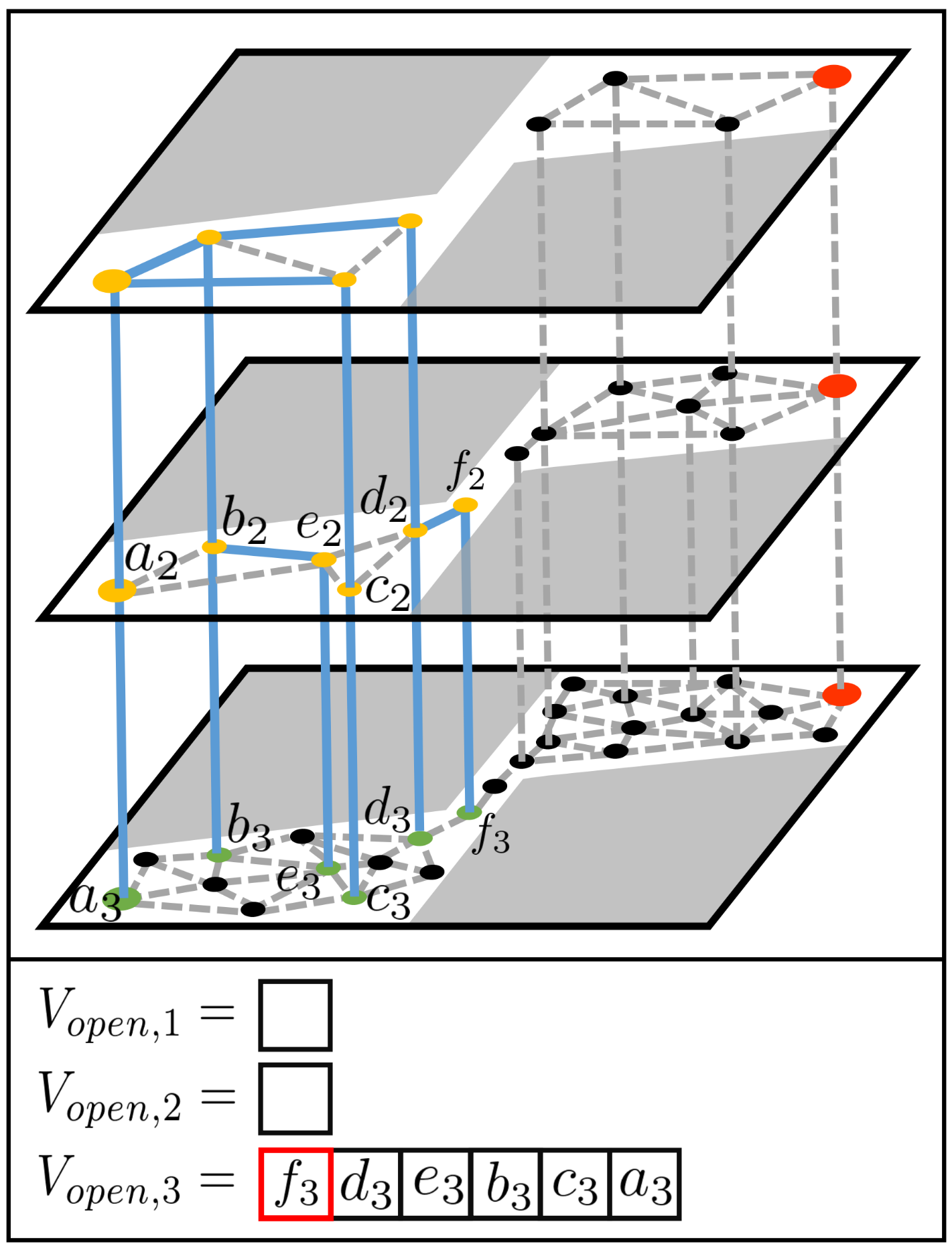}}
  
  \subfloat[$f_3$ to $h_3$ were iteratively expanded and closed. 
  As $h_3$ has a counterpart sample $h_2$ with a sparser resolution, $p$ decreases to 2.]{\includegraphics[width = 0.24\textwidth]
  {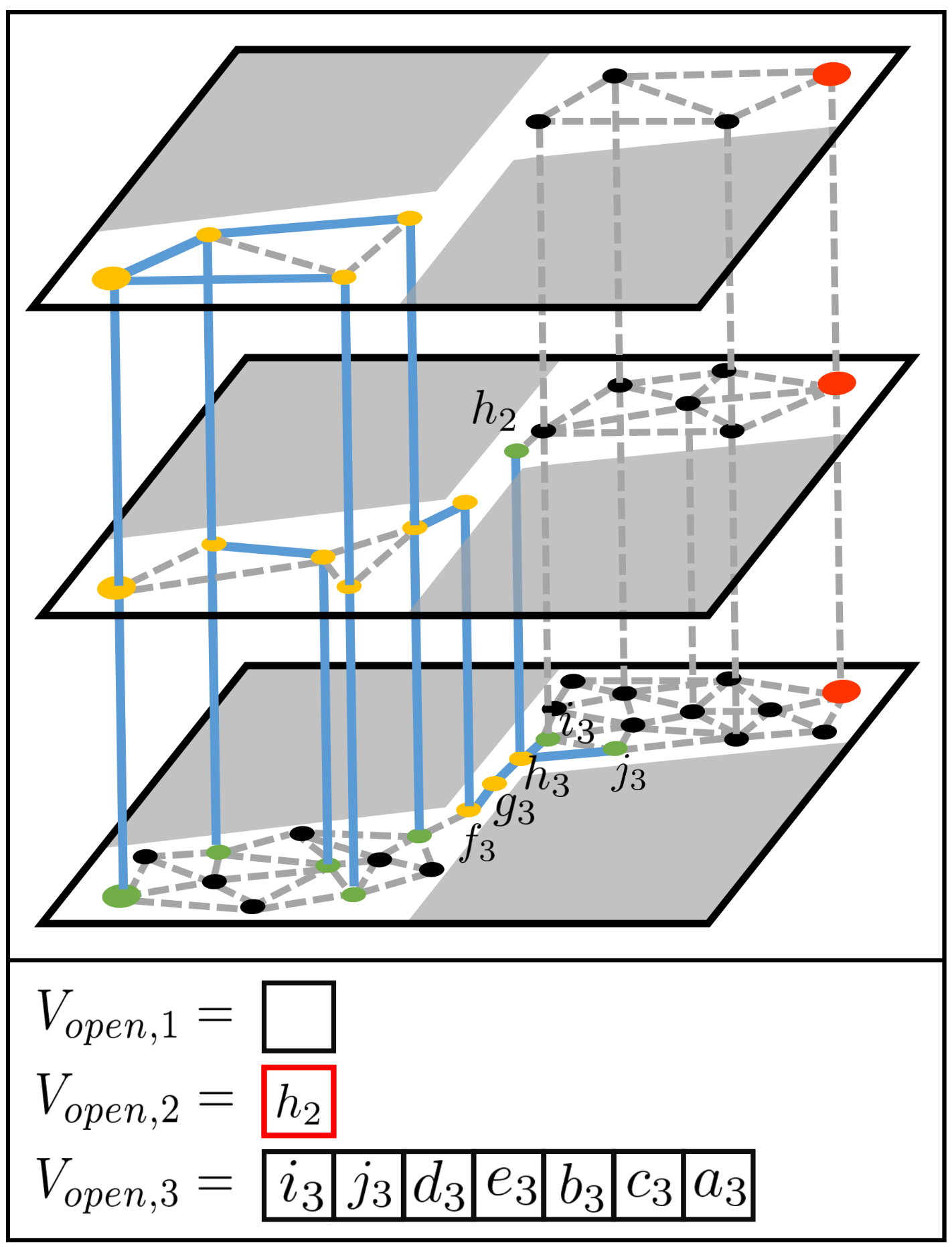}}\hfil
  \subfloat[$h_2$ and $i_2$ were expanded and closed. As $i_1$ was opened and connected to the tree, $p$ decreases to 1.]{\includegraphics[width = 0.24\textwidth]
  {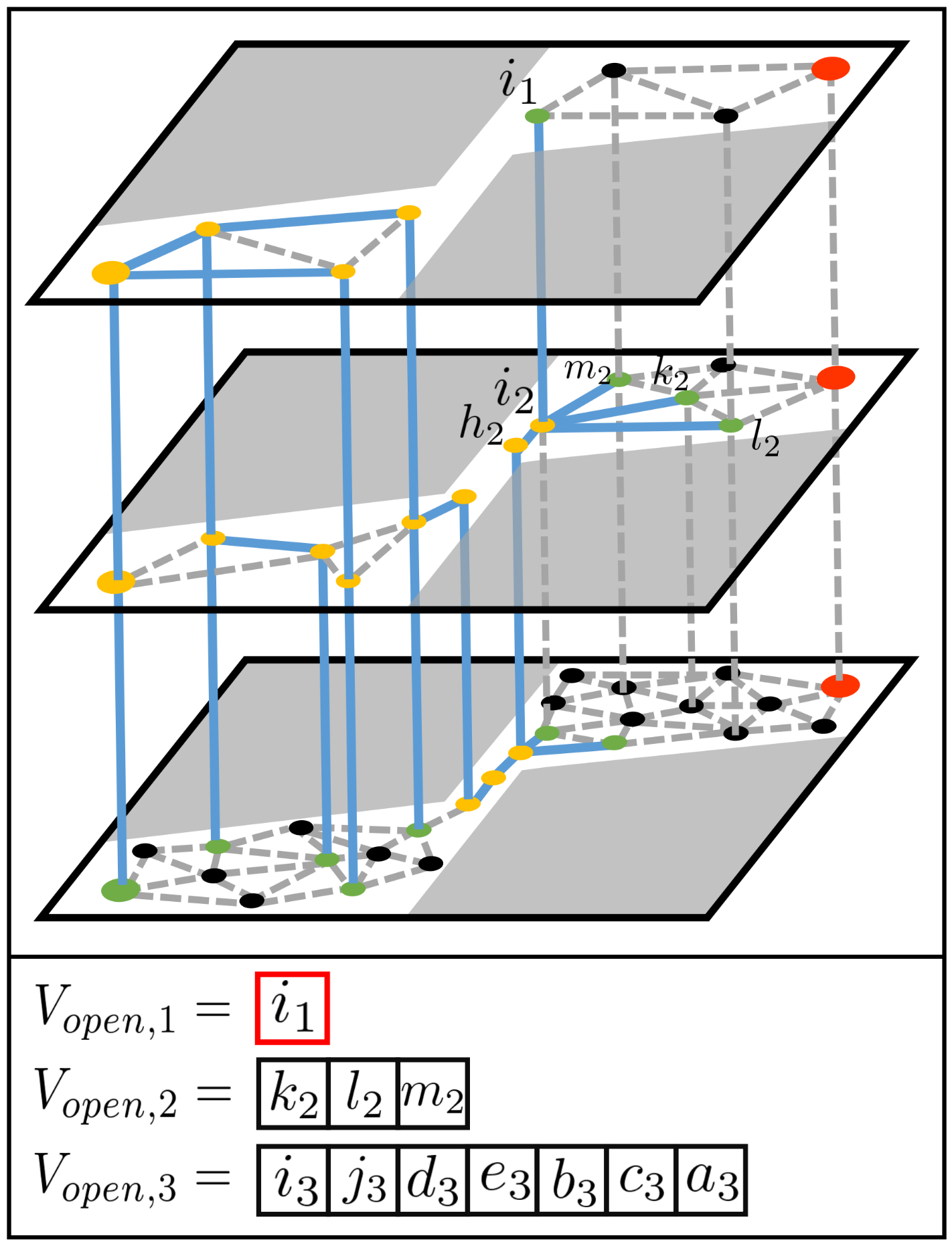}}\hfil
  \subfloat[$i_1$ and \text{\color{red}$l_1$} were expanded and closed. The goal sample is connected to the tree, and a solution path can be found.]{\includegraphics[width = 0.24\textwidth]
  {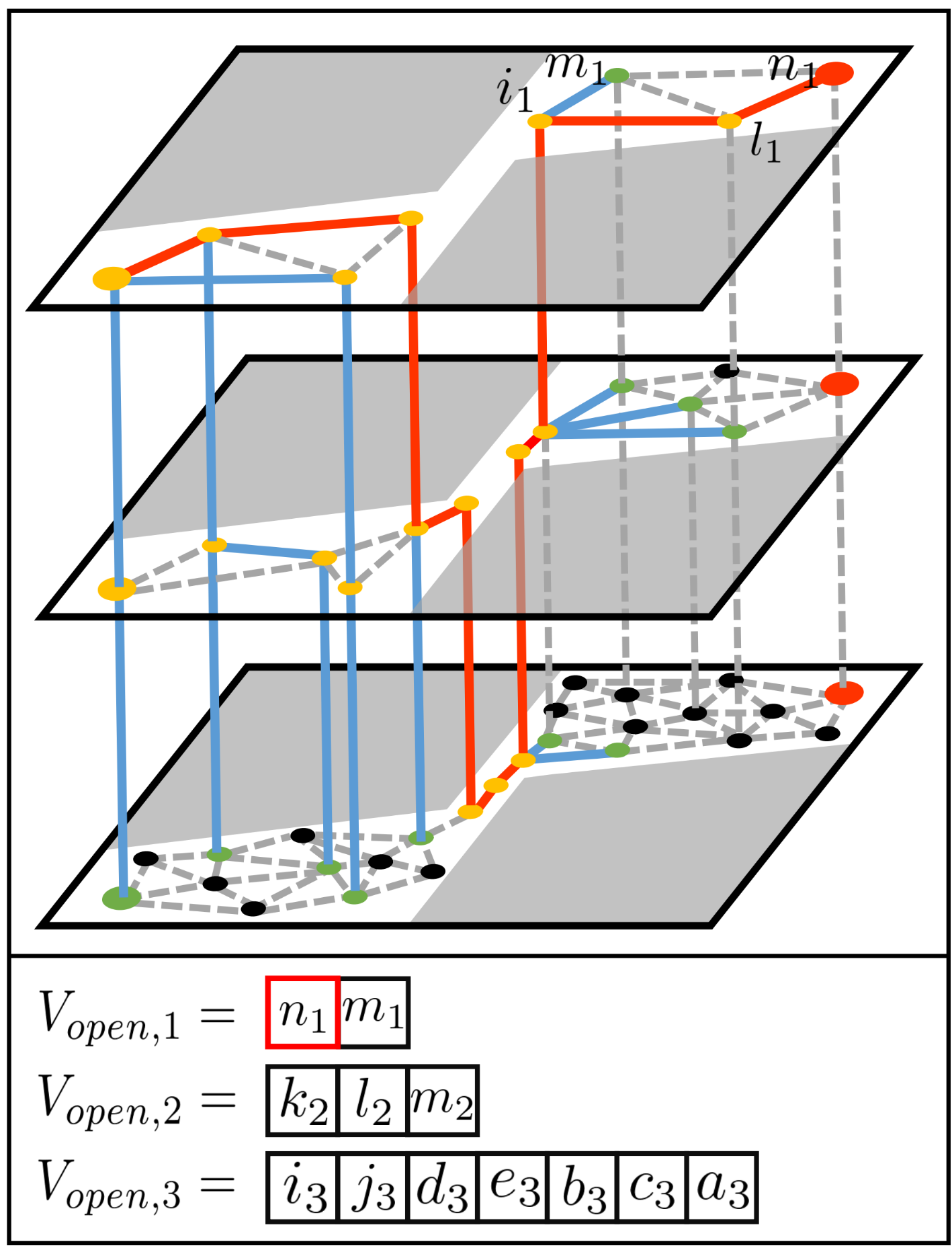}}

\caption{Visualization of the step-by-step planning process of MRFMT$^*$ for the motion planning problem in Fig.\ref{fig:2D_narrow_corridor}. 
The black, green, and yellow dots represent unvisited, open, and closed samples. 
The dashed grey lines depict the edges of the underlying search graph. 
Specifically, the vertical edges are the zero-cost edges connecting samples in adjacent layers representing the same configuration. 
The blue lines represent the edges of the search tree.
The sample to expand, $z$, is highlighted by red rectangles.
}
\label{fig:Steps}
\end{figure*}

The search process of MRFMT$^*$ for the narrow-corridor motion planning problem presented in Fig.\ref{fig:2D_narrow_corridor} is visualized in Fig.\ref{fig:Steps}. 
MRFMT$^*$ is initialized with three random sample sets. 
Each random sample set forms an underlying search graph, as shown in Fig.\ref{fig:2D_narrow_corridor}. 
Thanks to the defined zero-cost edges between samples representing the same configuration, we can visualize the underlying search graph of MRFMT$^*$ as a three-layered graph.
The vertical edges have samples with the same configuration at the ends and, therefore, can be lazily skipped for edge evaluation during the planning procedure.
Compared with the search result of FMT$^*$ as shown in Fig.\ref{fig:2D_narrow_corridor}, MRFMT$^*$ finds a solution with much fewer edge evaluations (MRFMT$^*$ evaluates 16 edges during the planning, while FMT$^*$ evaluates 24 edges.) and neighbor queries (Because MRFMT$^*$ searches neighboring samples mainly among sparse random sample sets, while FMT$^*$ constantly searches among the densest random sample set.), showing significant benefits in search cost reduction.

\subsection{Bidirectional Version - BMRFMT$^*$}
\label{BMRFMT_detail}
BMRFMT$^*$ is presented in Algorithm \ref{BMRFMT} as simple modifications of MRFMT$^*$, with changes highlighted in blue.
Instead of constructing only one search tree,
BMRFMT$^*$ constructs a forward search tree $\mathcal{T} = (V, E, V_{unvisited}, \{V_{open, l}\}_{l=1,\cdots,L})$ rooted at $x_{init}^1$ and a backward search tree $\mathcal{T}' = (V', E', V_{unvisited}', \{V_{open, l}'\}_{l=1,\cdots,L})$ rooted at $x_{goal}^1 \in \mathcal{X}_{goal}$ (line \ref{initialize}-\ref{BMRFMT:initialize_}, Algorithm \ref{BMRFMT}), which are expanded simultaneously during planning.
We use $z$ and $z'$ as the pointers to the wavefront node of $\mathcal{T}$ and $\mathcal{T}'$, respectively, and $p$ and $p'$ as the sparsity level of $z$ and $z'$, respectively.
$x_{meet}$ is the lowest-cost candidate sample for tree connection, which is updated in the \FuncSty{EXPAND} procedure (lines \ref{Expand:UpdateXmeetBegin}-\ref{Expand:UpdateXmeetEnd}, Algorithm \ref{Expand}). 
BMRFMT$^*$ checks whether the expansion is successful at each iteration in line \ref{BMRFMT:checkExpansionSuccessful}, Algorithm \ref{BMRFMT}.
If the expansion of the current tree fails, it swaps the search trees if the other tree has at least one sample in its search wavefront (line \ref{BMRFMT:SwapTreeIfCurrentFailsBegins}-\ref{BMRFMT:SwapTreeIfCurrentFailsEnds}, Algorithm \ref{BMRFMT}), or breaks if the other tree has no wavefront nodes either (line \ref{BMRFMT:break}, Algorithm \ref{BMRFMT}). 
In the opposite case, i.e., the expansion is successful, it swaps the search tree if the other tree has at least one sample in its search wavefront (line \ref{BMRFMT:SwapIfTheOtherTreeIsNotEmptyBegins}-\ref{BMRFMT:SwapIfTheOtherTreeIsNotEmptyEnds}, Algorithm \ref{BMRFMT}) or continues expanding the current tree until all samples are closed.
A solution path can be found by tracing along $\mathcal{T}$ and $\mathcal{T}'$ from $x_{meet}$, provided $x_{meet}$ is not a null pointer (line \ref{BMRFMT:track_solution_path_begins}-\ref{BMRFMT:track_solution_path_ends}, Algorithm \ref{BMRFMT}).

\section{Analysis}\label{Section:Analysis}

This section provides a thorough analysis of the critical properties of the MRFMT$^*$ algorithm, demonstrating its correctness in Theorem \ref{correctness} and asymptotic optimality in Theorem \ref{AO}. 

\begin{lemma}
\label{ExpandOnce}
MRFMT$^*$ expands a sample at most once throughout the entire planning procedure.
\end{lemma}

\begin{proof}
By construction, any sample to expand except for $x_{init}$ must be in the unvisited neighboring sample set of the lowest-cost open sample $z$, $Z_{near}$ (line \ref{unvisited}, Algorithm \ref{Expand}). 
Upon expansion, a successfully expanded sample is removed from $V_{unvisited}$ (line \ref{remove_z_from_open}, Algorithm \ref{Expand}), and therefore will never be included in $Z_{near}$ in future expansions. 
Consequently, each sample can be expanded at most once before MRFMT$^*$ finds a solution.
\end{proof}

\begin{theorem}[Probability Completeness of MRFMT$^*$]
\label{correctness}
\textcolor{black}{Let $(\mathcal{X}_{free}, x_{init}, \mathcal{X}_{goal})$ be a robustly feasible path planning problem. Let $n_L$ denote the number of samples in the finest resolution layer of MRFMT$^*$. 
Then, as $n_L \to \infty$, the probability that MRFMT$^*$ finds a feasible path converges to one:}
\begin{equation}
\lim_{n_L \to \infty} \mathbb{P} \left( V\cap \mathcal{X}_{goal}\right) = 1,
\end{equation}
where $V$ is the set of samples connected to the search tree $\mathcal{T}$.
\end{theorem}

\begin{proof}
From Lemma \ref{ExpandOnce}, we can infer that MRFMT$^*$ either terminates with at least one expanded sample in $\mathcal{X}_{goal}$ and returns a solution path, or exits the search iteration without a solution after exhausting all $V_{open, l}$ for $l=1,\cdots, L$
In the worst case, where all random sample sets with lower resolutions do not contain a solution,
MRFMT$^*$ performs an FMT$^*$ search over the finest resolution.
Hence, the probability that MRFMT$^*$ returns a feasible solution is lower bounded by the probability that FMT$^*$ finds a feasible solution by searching over the finest resolution sample set.
According to \cite{FMT}, when the number of random samples is sufficiently large, the probability that FMT$^*$ finds a feasible solution approaches 1. Therefore, as the number of samples at the finest resolution $n_L\rightarrow\infty$, the probability that MRFMT$^*$ finds a feasible solution also approaches 1, making it probabilistically complete with respect to the finest resolution samples.
\end{proof}

\textcolor{black}{Therefore, to ensure a high planning success rate, a sufficiently large random sample set should be initialized, albeit at the expense of increased search cost.}

\begin{theorem}[Asymptotical Optimality of MRFMT$^*$]
\label{AO}
Let $(\mathcal{X}_{free}, x_{init}, \mathcal{X}_{goal}, c)$ be a $\xi$-robustly feasible path planning problem in d dimensions, where $\xi>0$, $\mathcal{X}_{goal}$ is $\xi$-regular, and the cost function $c$ admits a robustly optimal solution with finite cost. 
\textcolor{black}{Then, as $n_1 \to \infty$, the probability that MRFMT$^*$ returns a path with cost arbitrarily close to $c^*$ converges to one:
\begin{equation}
\lim_{n_1 \to \infty} \mathbb{P} \left( c(\pi_{n_1}) \to c^* \right) = 1,
\end{equation}
where $\pi_{n_1}$ denotes the solution returned by MRFMT$^*$ using $n_1$ samples.}
\end{theorem}

\begin{proof}
Let $c^*$ denote the cost of the optimal solution, and $c^{ALG}$ denote the cost of the path returned by a motion planning algorithm \FuncSty{ALG}.
To solve the motion planning problem, MRFMT$^*$ applies an FMT$^*$ search over the coarsest resolution sample set until the coarsest resolution sample set has no samples to expand. 
From Theorem 4.1 in \cite{FMT}, the cost of the path returned by FMT$^*$ with $n$ random samples using the radius in Equation 3 of \cite{FMT}, denoted as $c^{FMT^*}$, satisfies $\mathbb{P}(c^{FMT^*}>(1+\varepsilon)c^*)=0$ as $n\rightarrow\infty$ for all $\varepsilon>0$.
Thus, the probability that MRFMT$^*$ biases towards denser resolution sample sets approaches 0 as $n_1$ approaches $\infty$ when MRFMT$^*$ is equivalent to FMT$^*$.
Therefore, 
$$
\mathbb{P}(c^{MRFMT^*}>(1+\varepsilon)c^*)=0
$$
as $n_1\rightarrow\infty$ for all $\varepsilon>0$, proving that MRFMT$^*$ is asymptotically optimal.
\end{proof}

\textcolor{black}{
Therefore, when the sparser random sample set is sufficiently large, the solution cost improves, albeit at the expense of higher search cost.
}

\begin{figure*}
\centering
  \subfloat[$\mathbb{SE}(2)$ bug trap]{\includegraphics[width = 0.27\textwidth]{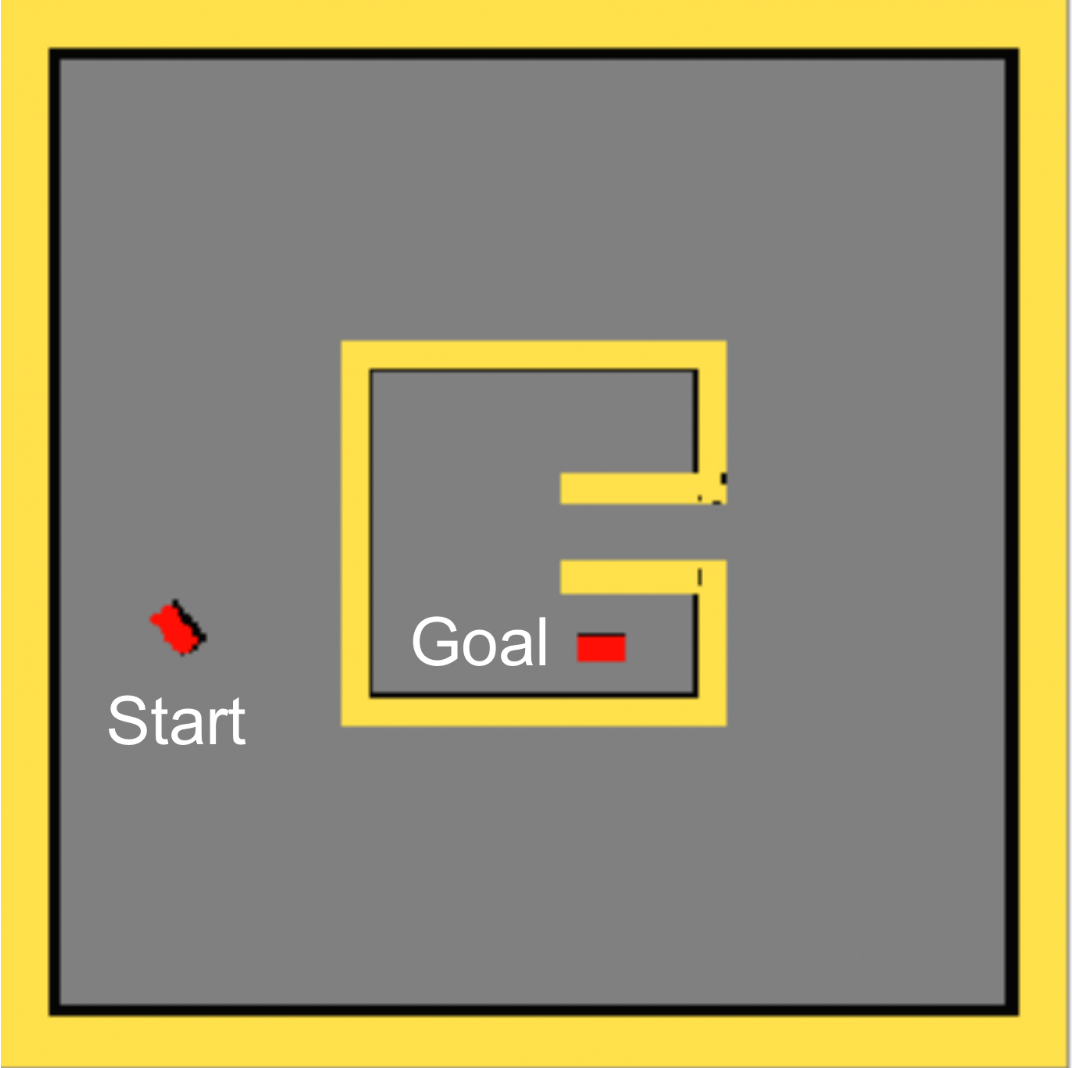}}\hfil
  \subfloat[$\mathbb{SE}(3)$ apartment]{\includegraphics[width = 0.29\textwidth]{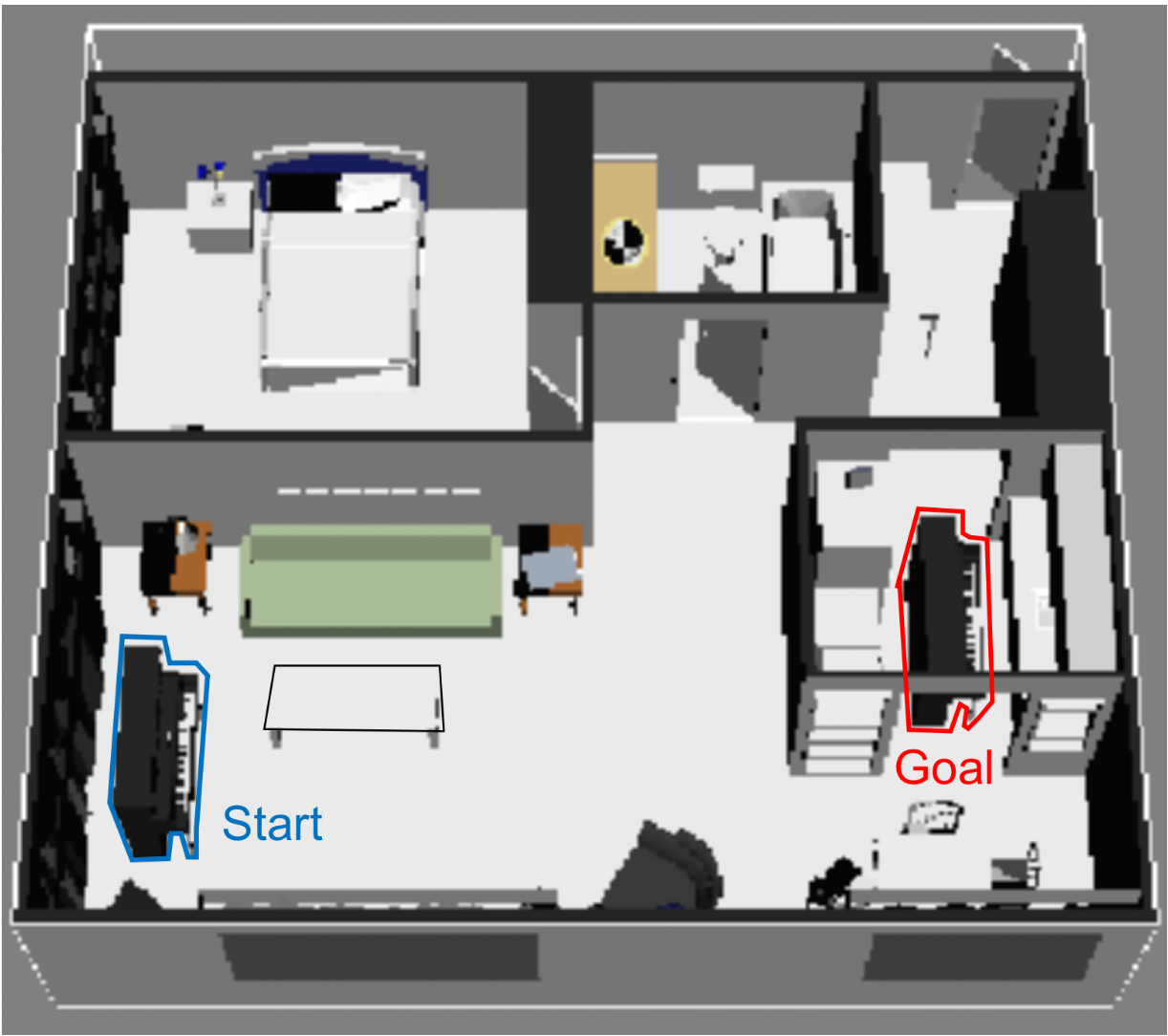}}\hfil
  \subfloat[14-DoF Link Robot]{\includegraphics[width = 0.26\textwidth]{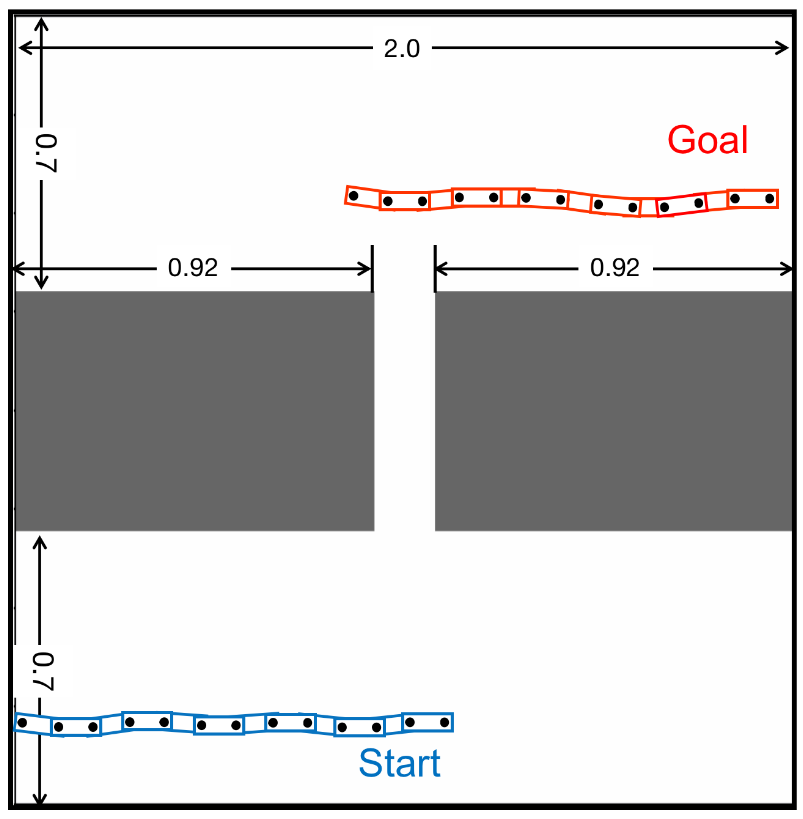}}
\caption{Descriptions of the $\mathbb{SE}(2)$, $\mathbb{SE}(3)$, and $\mathbb{R}^{14}$ rigid-body motion planning problems.}
\label{fig:SimulationScene}
\end{figure*}

\section{Simulations}
\label{Section:Simulation}

\subsection{Simulation Setup}
In this section, we compare MRFMT$^*$ and BMRFMT$^*$ with several sampling-based motion planning algorithms, specifically PRM$^*$\cite{PRM}, RRT$^*$\cite{RRTstar}, BIT$^*$\cite{BIT}, SPARS2\cite{SPARS2}, FMT$^*$\cite{FMT}, and BFMT$^*$\cite{BFMT}, to numerically investigate the advantages of MRFMT$^*$ and BMRFMT$^*$. 
To ensure a fair comparison, each planning algorithm was tested using the Open Motion Planning Library (OMPL) v1.6.0~\cite{OMPL}.
We considered three motion planning problems from the OMPL's test suite:
\begin{itemize}
    \item the bug trap problem in $\mathbb{SE}(2)$ as shown in Fig.\ref{fig:SimulationScene}(a);
    \item the piano movers' problem in $\mathbb{SE}(3)$ as shown in Fig.\ref{fig:SimulationScene}\text{\color{red}(b)}.
\end{itemize}
Besides, we also consider the following problem to investigate the planners' performance in very high-dimensional configuration spaces:
\begin{itemize}
    \item the movable link robot problem in $\mathbb{R}^{14}$ as shown in Fig.\ref{fig:SimulationScene}\text{\color{red}(c)}.
\end{itemize}
The link robot has 12 links and can move freely in the x-y plane. 
In each case, dynamic constraints were neglected, and arc length was used as the solution cost for all problems.

For the benchmarking algorithms,  we used the default OMPL settings,
except that we used heuristics for all algorithms and did not extend the graph of FMT$^*$ and BFMT$^*$.
We ensured that MRFMT$^*$ and BMRFMT$^*$ used the same tuning parameters and configurations as the benchmarking algorithms whenever possible.
To compare the quality between incremental or "anytime" planners (i.e., RRT$^*$, SPARS2, PRM$^*$, and BIT$^*$) and non-incremental planners (i.e., FMT$^*$, BFMT$^*$, MRFMT$^*$, and BMRFMT$^*$, which generate solutions via sample batches), we varied the number of free configuration samples $N$ taken during initialization for the non-incremental planners. 
This variation serves as a proxy for execution time. 
Specifically, $N$ ranged from 200 to 10,000 for the $\mathbb{SE}(2)$ problems, from 1,000 to 30,000 for the piano movers' problem, and from 4,000 to 40,000 for the movable link robot problem.
For MRFMT$^*$ and BMRFMT$^*$, we used a linearly increasing sequence to allocate the number of samples for each random sample set. 
Specifically, the number of nodes in the $l$th random sample set is given by $n_l = \lfloor lN/L \rfloor$. 
We used $L=4$ for the $\mathbb{SE}(2)$ problems and $L=6$ for the $\mathbb{SE}(3)$ and the $\mathbb{R}^{14}$ problems as empirically the chosen parameters fit comfortably in memory and was able to solve our scenarios.
The maximum memory was limited to 4096 MB for all planners.

\subsection{Simulation Results and Discussions}
\begin{figure*}[h]
\centering
  \subfloat{\includegraphics[width = 0.33\textwidth]{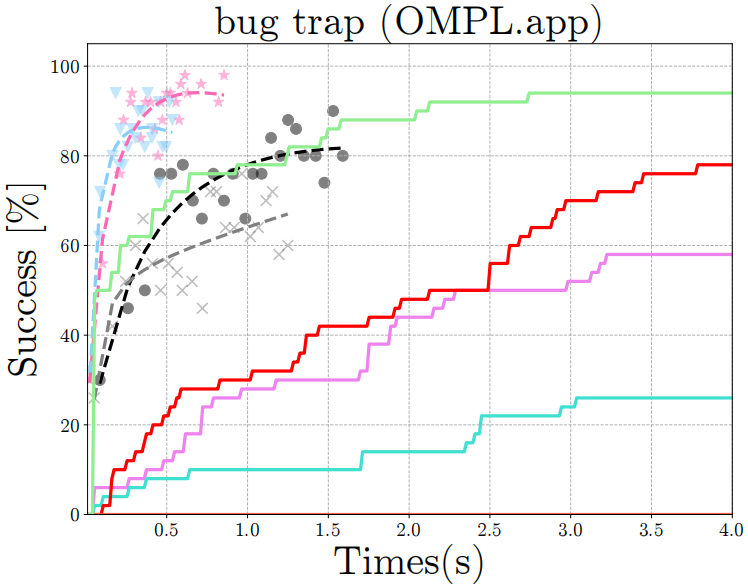}}\hfil
  \subfloat{\includegraphics[width = 0.33\textwidth]{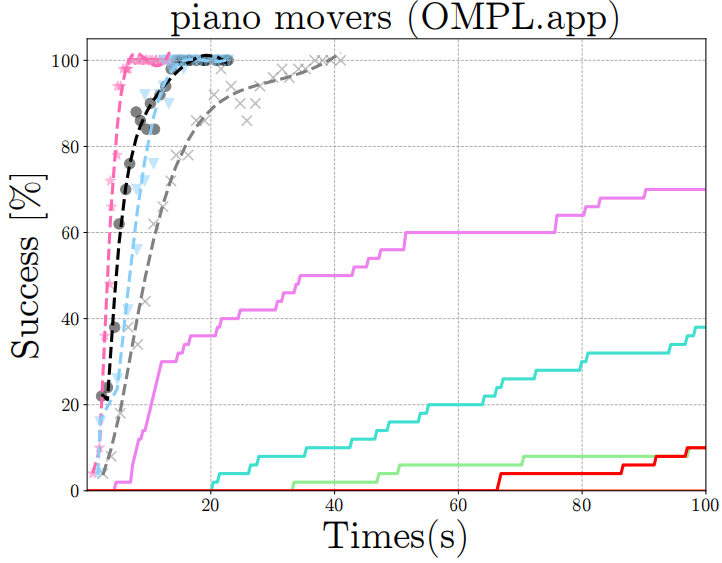}}
  \subfloat{\includegraphics[width = 0.33\textwidth]{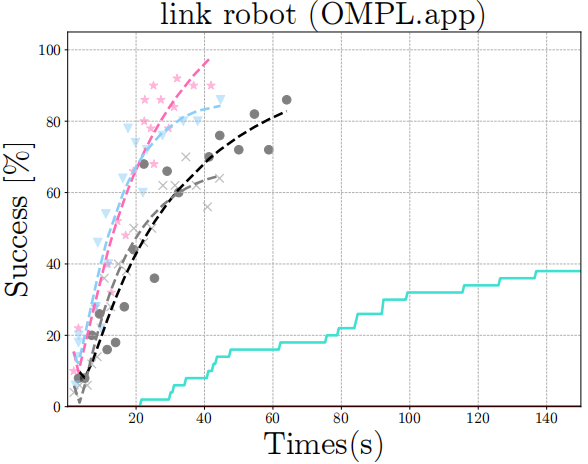}}
  
  \subfloat{\includegraphics[width = 0.33\textwidth]{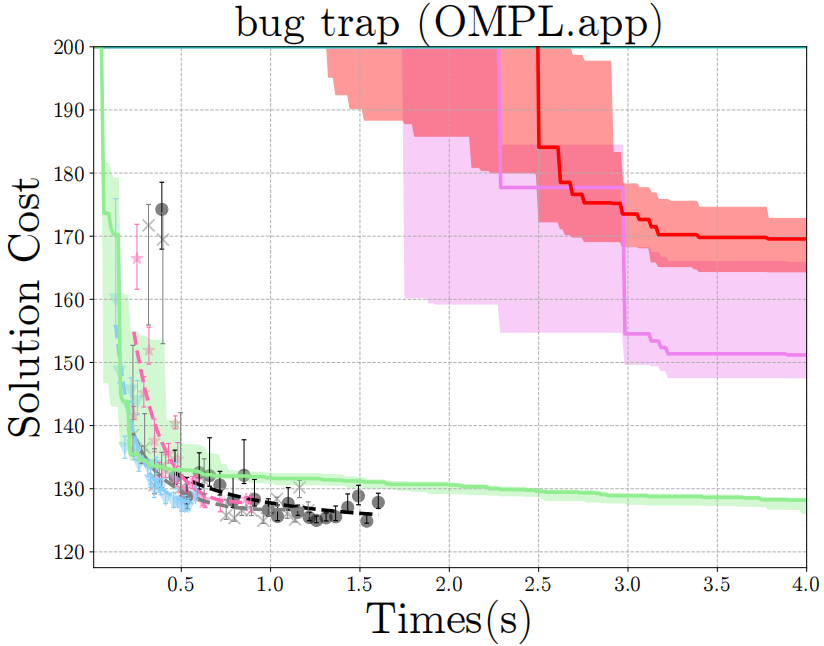}}\hfil
  \subfloat{\includegraphics[width = 0.33\textwidth]{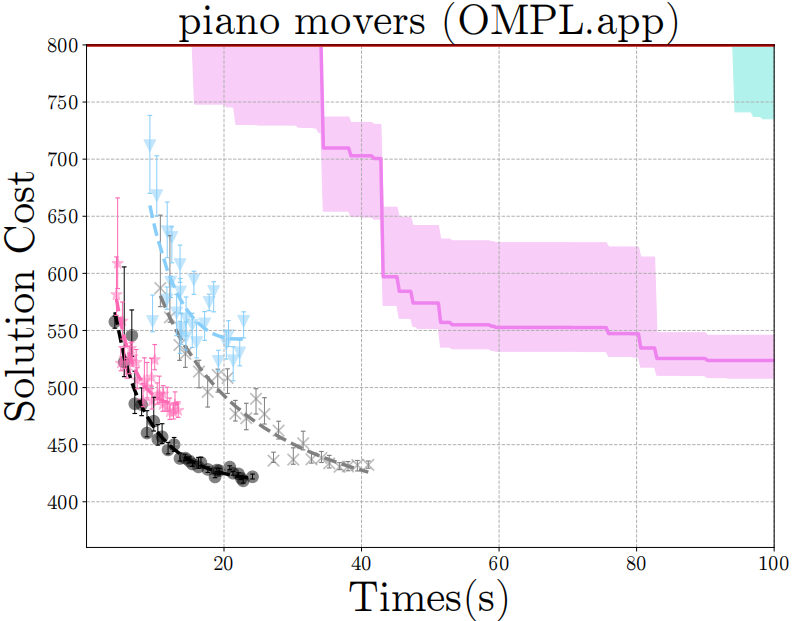}}
  \subfloat{\includegraphics[width = 0.33\textwidth]{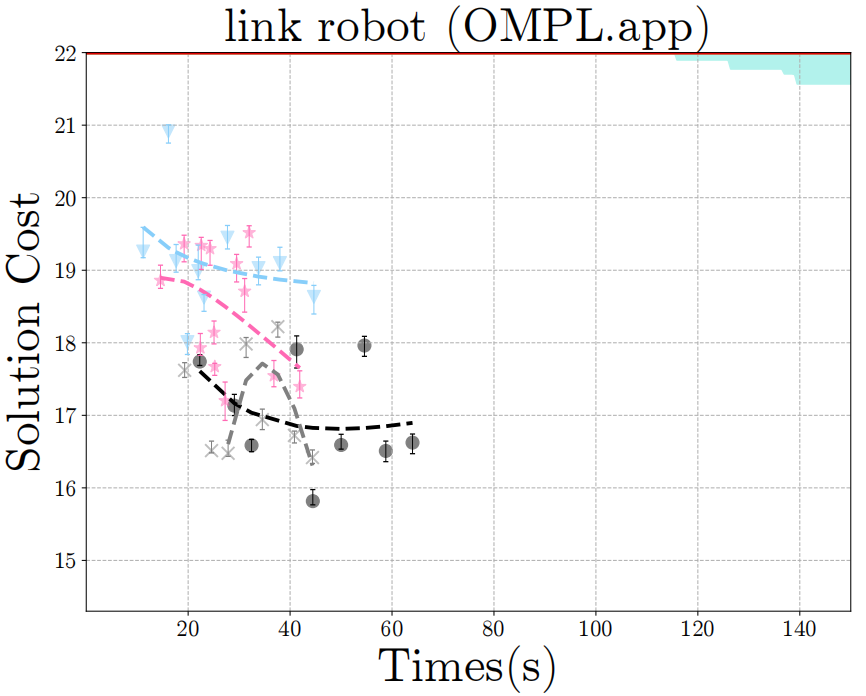}}
  
  \subfloat{\includegraphics[width = 0.75\textwidth]{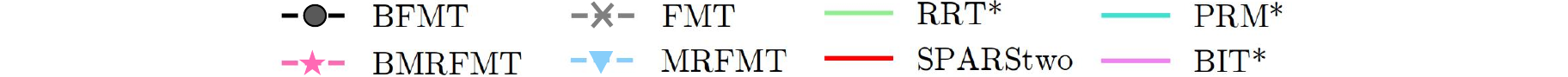}}
\caption{Planner performance versus time. Each planner was run 50 different times. The median path length is plotted versus run time/sample count for each planner, with unsuccessful trials assigned infinite cost.
\textcolor{black}{The median values are plotted with error bars denoting a non-parametric 95$\%$ confidence interval on the median. The dashed lines are regression lines fitted to the points associated with a given planner.}
}
\label{fig:ResultVersusTime}
\end{figure*}

\begin{figure*}[h]
\centering
  \subfloat{\includegraphics[width = 0.33\textwidth]{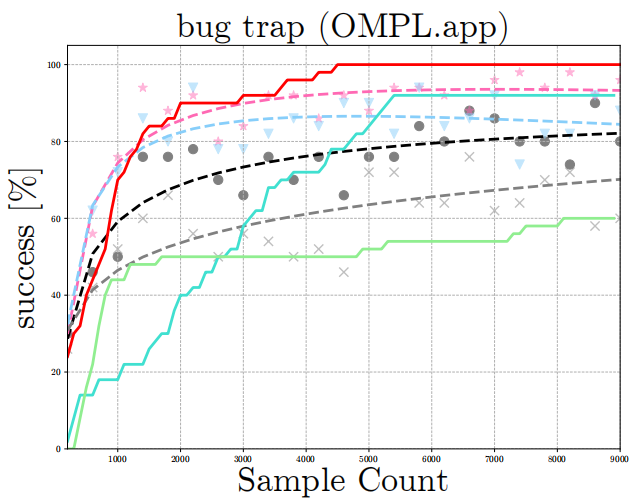}}\hfil
  \subfloat{\includegraphics[width = 0.33\textwidth]{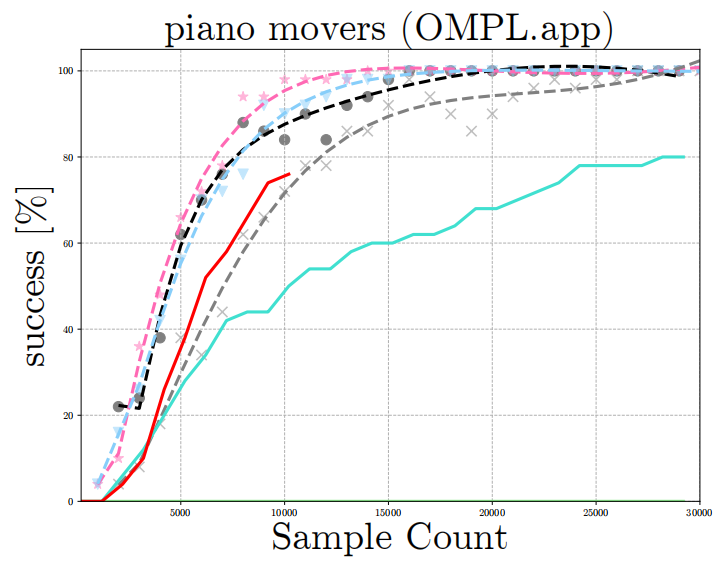}}\hfil
  \subfloat{\includegraphics[width = 0.33\textwidth]{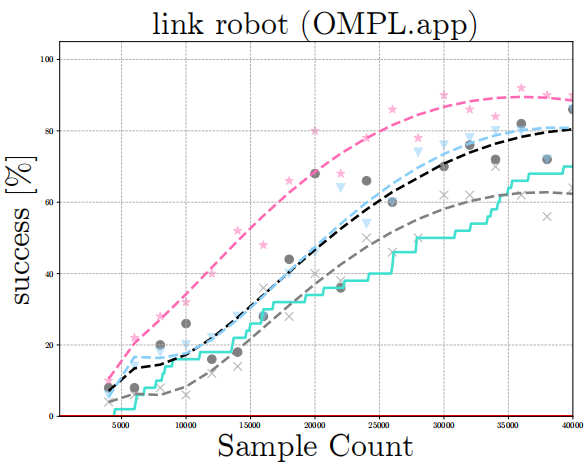}}
  
  \subfloat{\includegraphics[width = 0.33\textwidth]{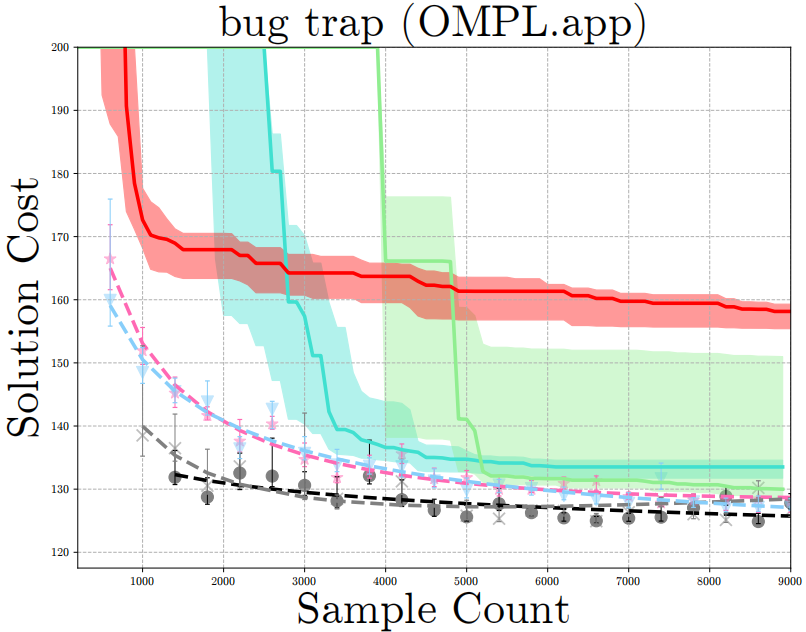}}\hfil
  \subfloat{\includegraphics[width = 0.33\textwidth]{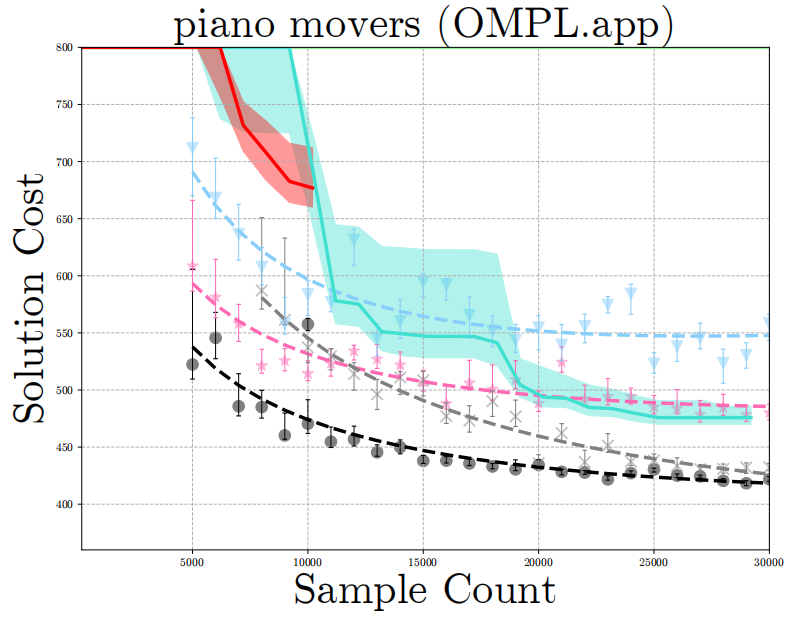}}\hfil
  \subfloat{\includegraphics[width = 0.33\textwidth]{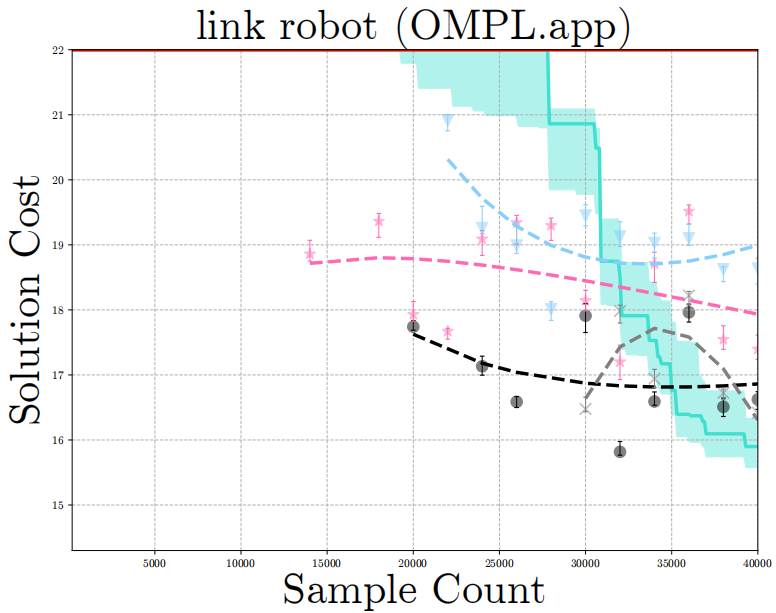}}
  
  \subfloat{\includegraphics[width = 0.33\textwidth]{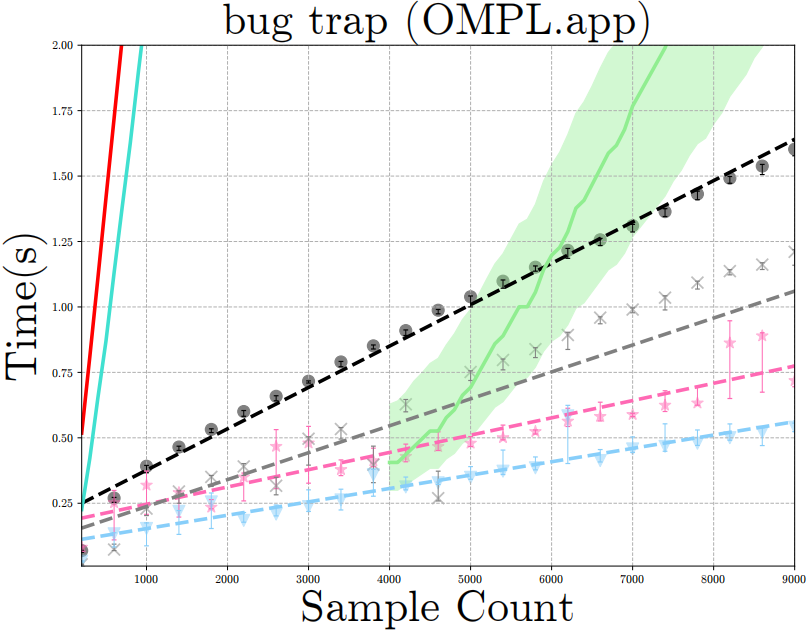}}
  \subfloat{\includegraphics[width = 0.33\textwidth]{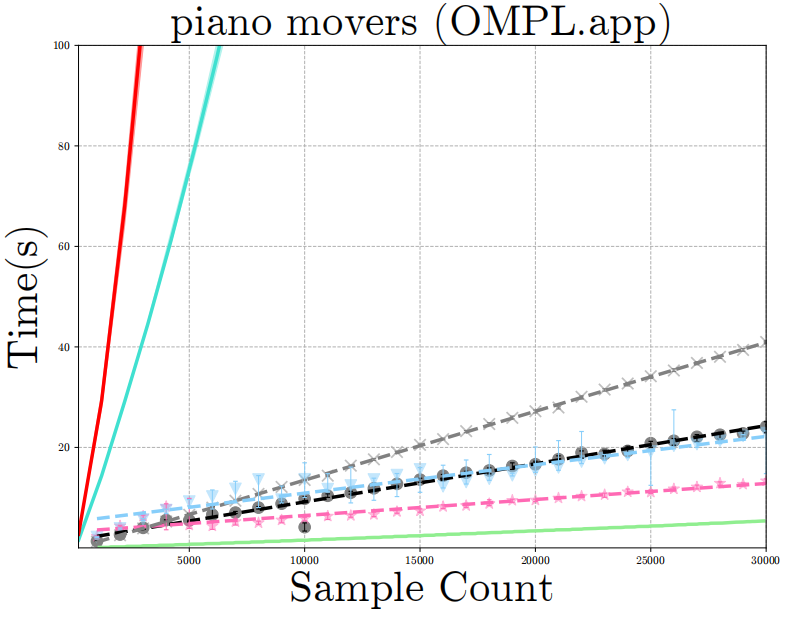}}\hfil
  \subfloat{\includegraphics[width = 0.33\textwidth]{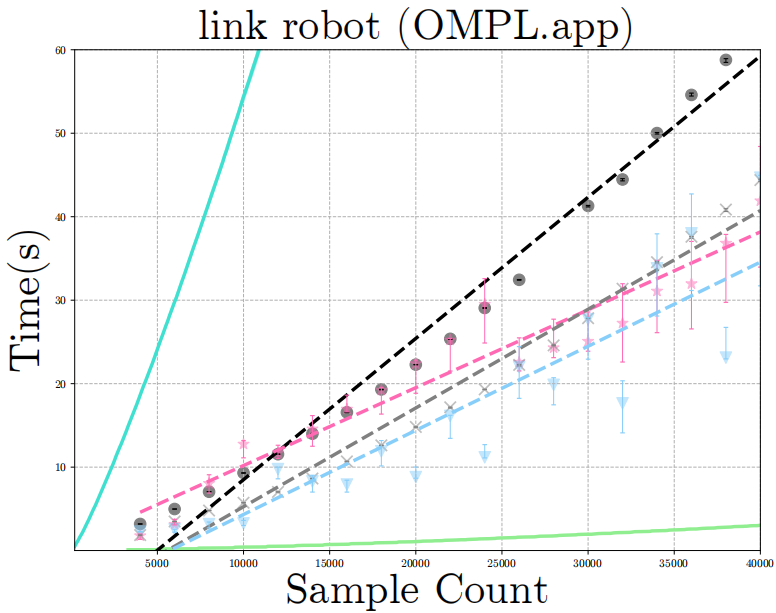}}

  \subfloat{\includegraphics[width = 0.75\textwidth]{Figs/label.pdf}}
\caption{Planner performance versus sample count. Each planner was run 50 different times. The median path length is plotted versus run time/sample count for each planner, with unsuccessful trials assigned infinite cost. 
\textcolor{black}{The median values are plotted with error bars denoting a non-parametric 95$\%$ confidence interval on the median.
The dashed lines are regression lines fitted to the points associated with a given planner.}
}
\label{fig:ResultVersusSample}
\end{figure*}

We present the simulation results for success rate and solution cost versus time in Fig. \ref{fig:ResultVersusTime}. 
Each point of a non-incremental planner represents the results of 50 simulations with the exact sample count.
Since sample count is the primary parameter for non-incremental planners, we also examine its impact on planning performance in Fig. \ref{fig:ResultVersusSample}, which illustrates the simulation results as a function of sample count for non-incremental planners and as a function of planning iteration for incremental planners. 

The simulation results in Fig. \ref{fig:ResultVersusTime} indicate that MRFMT$^*$ and BMRFMT$^*$ are the fastest in achieving high success rates across all problems,
indicating an improved convergence rate.
\textcolor{black}{Their success rates converge to 1 given larger planning time, validating their probability completeness.} Regarding solution cost, they converge to the optimal solution more quickly than all other planners, except for FMT$^*$ and BFMT$^*$. 
This is because they prioritize searching through sparser graph layers as long as they remain connected in free space, sacrificing optimality for increased efficiency. 
\textcolor{black}{ However, for simple problems such as the bug trap problem, MRFMT$^*$ and BMRFMT$^*$ converge to the optimal solution rapidly, even when compared with FMT$^*$ and BFMT$^*$.}
The success rate and time versus sample count in Fig.\ref{fig:ResultVersusSample} demonstrate that MRFMT$^*$ and BMRFMT$^*$ achieve higher success rates with lower time costs compared to FMT$^*$ and BFMT$^*$ at a given sample count. 
The higher success rate is attributed to their ability to simultaneously search through multiple graphs with varying densities and combine cross-layer subpaths in free space.
The speed advantage arises from the fact that FMT$^*$ and BFMT$^*$ perform a full expansion of every node, whereas MRFMT$^*$ and BMRFMT$^*$ primarily expand nodes on sparse graph layers when navigating through free space to quickly escape local minima, resorting to dense graph layers only when navigating through narrow passages.

Note that MRFMT$^*$ and BMRFMT$^*$ generate nearly identical cost-time curves in 2D problems. 
However, BMRFMT$^*$ outperforms MRFMT$^*$ in both solution cost and success rate in the Piano Mover's problem in $\mathbb{SE}(3)$ and the movable link robot problem in $\mathbb{R}^{14}$, where the volume of \textit{reachable configurations} around the goal is relatively small.
It suggests that the volume of reachable configurations significantly influences execution time. 
The relatively small volume implies that the backward tree of BMRFMT$^*$ expands its wavefront through fewer states than the forward tree of MRFMT$^*$. 
Additionally, the tree interconnection in the bidirectional case prevents the forward tree of BMRFMT$^*$ from growing too large compared to the unidirectional search of MRFMT$^*$, resulting in substantial computational cost savings in high-dimensional configuration spaces.

\section{Experiments}
\label{Section:Experiment}
\subsection{Experiment Setup}


\begin{figure*}
    \centering
   \subfloat[Start Position]{\includegraphics[width = 0.14\textwidth, height = 0.2\textwidth]{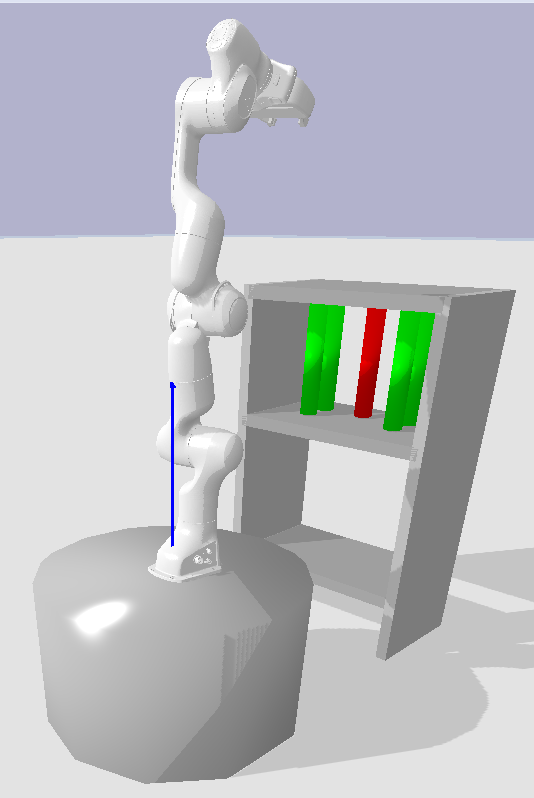}}\hspace{0.02\textwidth}
   \subfloat[Goal Position]{\includegraphics[width = 0.14\textwidth, height = 0.2\textwidth]{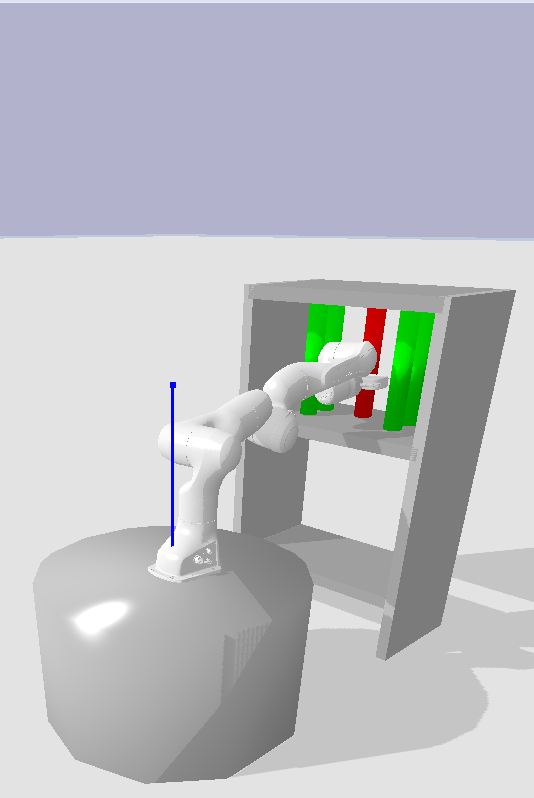}} \hspace{0.02\textwidth}
   \subfloat[A collision-free robot-arm trajectory found by MRFMT$^*$.]{\includegraphics[width = 0.64\textwidth]{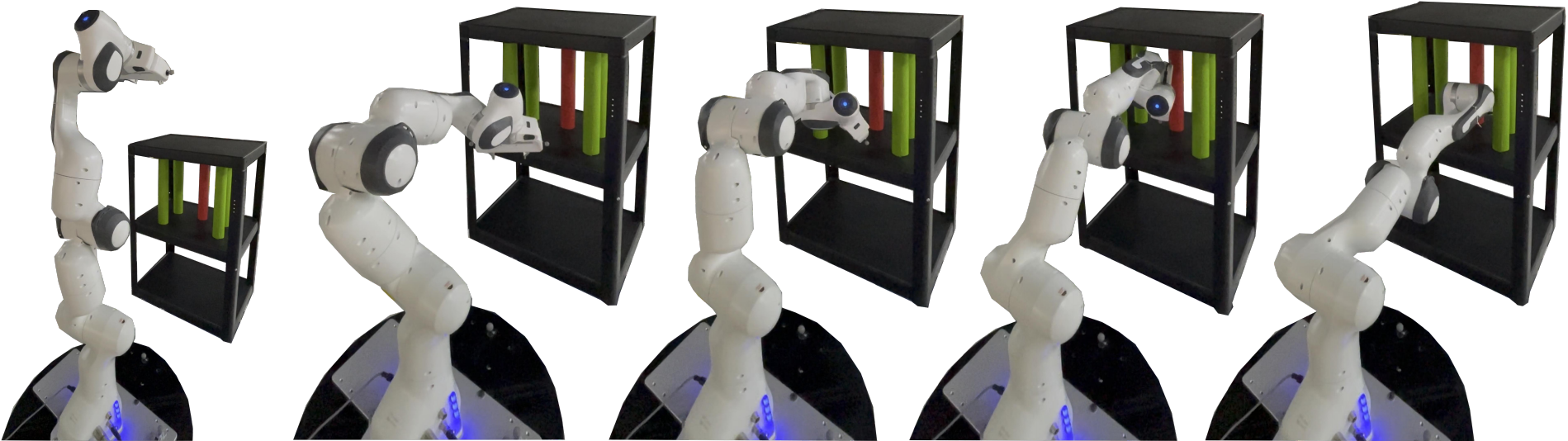}}
   
  \subfloat[Numerical results of the 7-DoF manipulation problem.]{\includegraphics[width = 0.8\textwidth]{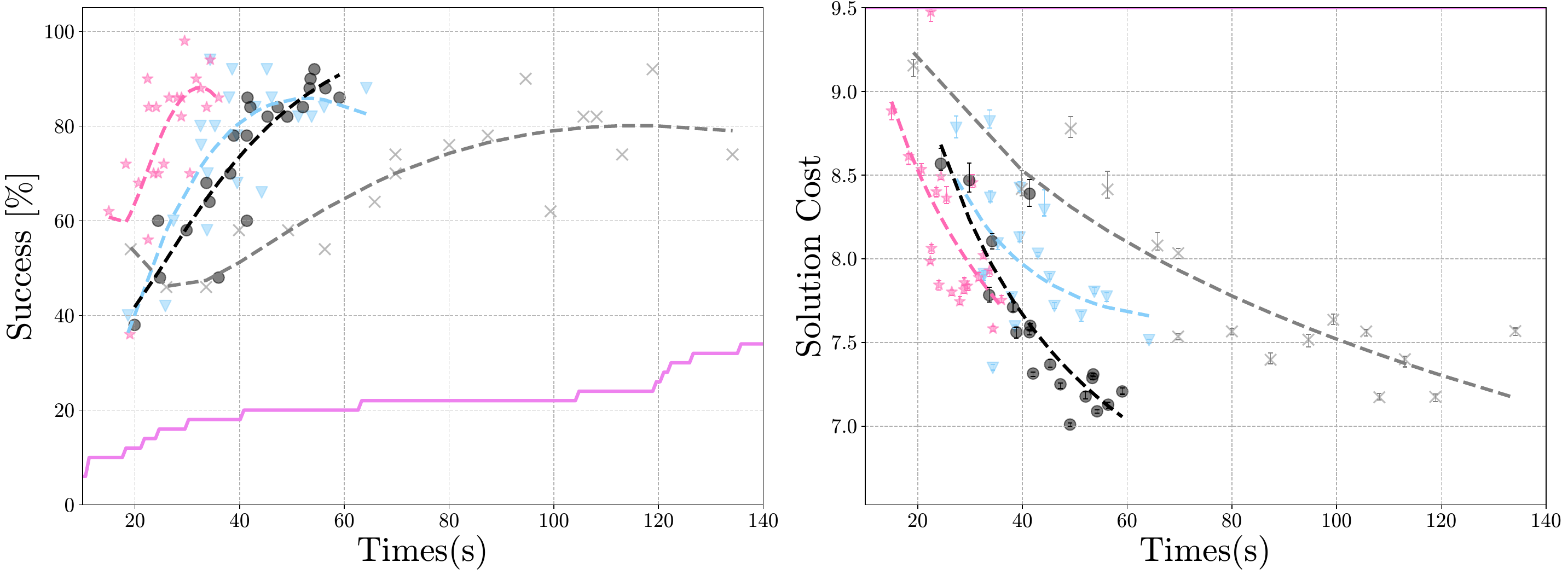}}
  
  \caption{Descriptions and results of a 7-DoF manipulation planning problem on the Franka Emika Panda robot. The robot arm is requested to grasp the red cylinder without removing the green cylinders, which is very challenging for traditional methods.}
  \label{fig:Franka_Problem}
\end{figure*}


  

We conducted experiments on the Franka Emika Panda robot equipped with a gripper to validate our proposed method in a realistic environment. 
In this experiment, the Franka robot was tasked with moving its gripper to retrieve a targeted red bottle located deep inside a shelf, while avoiding the green bottles positioned closely in front of it (see Fig.\ref{fig:Franka_Problem}(a)(b)). 
This scenario is particularly challenging due to the narrow passage available for the robot arm to maneuver. 
We compare MRFMT$^*$ and BMRFMT$^*$ with several OMPL benchmark planners.
The sample count for non-incremental planners was varied from 1,000 to 9,400. An r-disk graph with a radius of 3 was used as the underlying graph for all planners.
Both MRFMT$^*$ and BMRFMT$^*$ were evaluated with 4 resolutions. 

\subsection{Experiment Results and Discussions}
A solution provided by MRFMT$^*$ is displayed in Fig.\ref{fig:Franka_Problem}(c).
The numerical results from the 7-DoF manipulation experiments are summarized in Fig.\ref{fig:Franka_Problem}(d). 
Several planners are not displayed due to their low success rates within the allotted time.
The figure presents statistics over 50 runs for each planner. 
Notably, MRFMT$^*$ and BMRFMT$^*$ achieved the best performance among all planners, with BMRFMT$^*$ demonstrating reduced time to find feasible solutions.

\section{Discussions}
\label{Section:Discussions}

\subsection{Planning under Differential Constraints}
Motion planning for Driftless Control-Affine (DCA) systems is a classic problem in robotics.  
Some examples of DCA systems include
mobile robots with wheels that roll without slipping, 
Unmanned Aerial Vehicle (UAV) whose dynamics involve control inputs for thrust and moment generation, 
and humanoid Robotoid robots, like ASIMO or Atlas, when focusing on their walking and manipulation capabilities.
In \cite{DFMT}, a theoretical framework is proposed to assess optimality guarantees of sampling-based algorithms for planning under differential constraints. 
We can exploit the framework to extend MRFMT$^*$ and BMRFMT$^*$ to address the motion planning problem for DCA systems by
\begin{itemize}
    \item finding the neighboring samples with the same resolution by searching for the samples lying within the privileged coordinate box $Box^w(x, \gamma\left(\frac{\log{|n_l|}}{n_l}\right)^{\frac{1}{d}})$, where $Box^w(x, \epsilon)$ denotes the weighted box of size $\epsilon$ centered at $x\in\mathcal{X}$, $w={w_1,w_2,\cdots,w_d}$ denotes the weight vector at $x$, $\gamma$ is a constant defined by the user \footnote{For more detailed descriptions of the privileged coordinate box, please refer to \cite{DFMT,KinodynamicRRTstar}.}, and
    \item connecting samples by edges that are trajectories satisfying the differential constraints.
\end{itemize}
We show the planning results of MRFMT$^*$ in the Reeds-Sheep experiment conducted in the bug trap environment in Fig.\ref{fig:reedsheep_car}.
Admittedly, as FMT$^*$ and most other sampling-based algorithms, MRFMT$^*$ is not suitable for systems without direct access to the two-point boundary value solver.
We leave this problem to future research.
\begin{figure}
    \centering
    \includegraphics[width=0.48\linewidth]{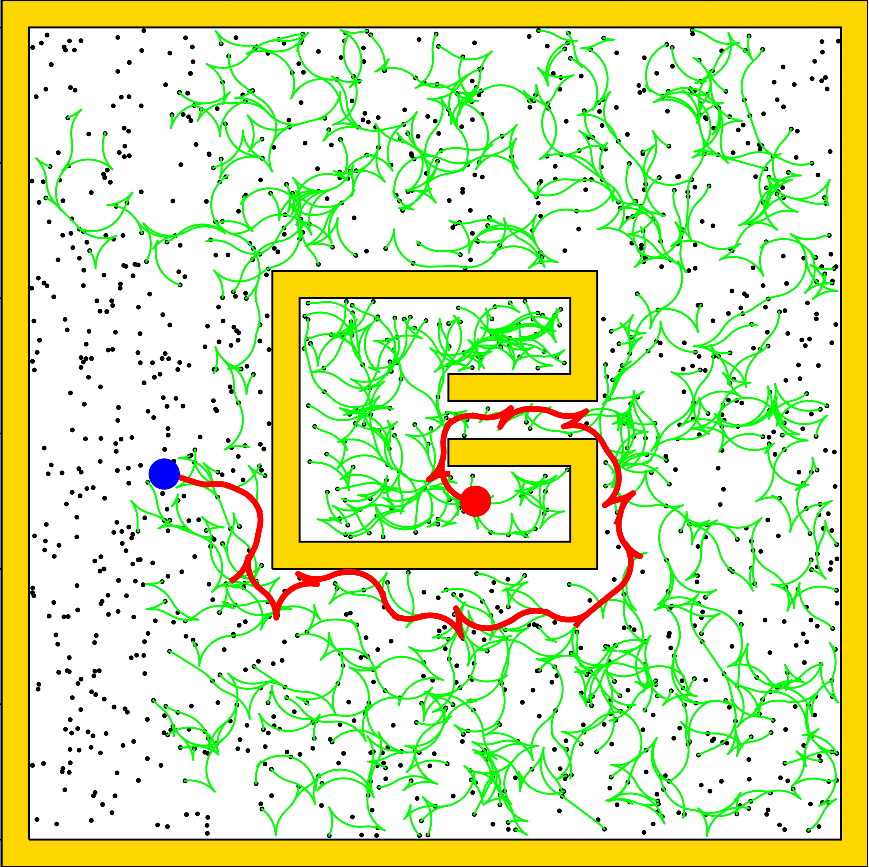}\hfill
    \includegraphics[width=0.48\linewidth]{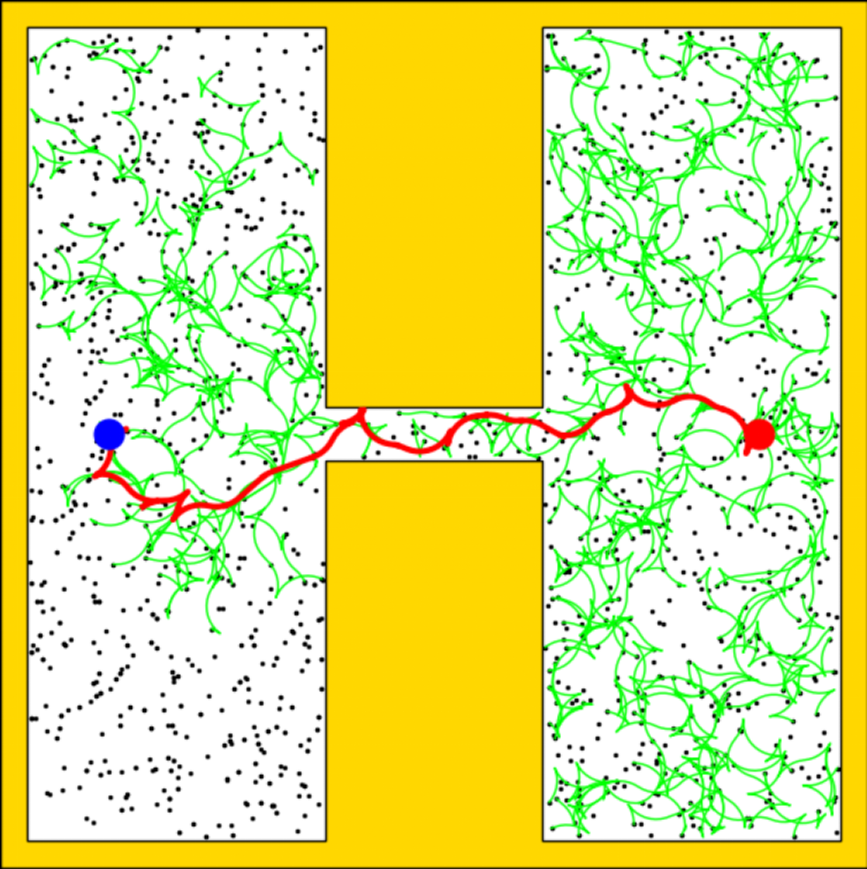}    
    \caption{The planning results of MRFMT$^*$ for a Reeds-Shepp car in different 2D environments with narrow passages.} 
    \label{fig:reedsheep_car}
\end{figure}

\subsection{Effects of the multi-resolution parameters}
Theorems \ref{correctness} and \ref{AO}, proved in Section \ref{Section:Analysis}, indicate that the structure of the underlying multi-layer graph significantly influences the performance of the proposed multi-resolution motion planners. 
To numerically investigate how performance varies with the underlying structure, we consider two types of graphs: one with linearly increasing layer density and another with exponentially increasing layer density. 
Specifically, for the graph with linearly increasing layer density, the number of nodes in the \( l \)-th layer is given by \( n_l = \lfloor \frac{l}{L}N \rfloor \).
In contrast, for the graph with exponentially increasing layer density, it is given by \( n_l = \lfloor \frac{1}{2^{L-l}}N \rfloor \). 


We compare MRFMT$^*$ across different graph structures, using FMT$^*$ as the benchmark for performance in solving motion planning problems in $\mathbb{SE}(2)$ and $\mathbb{SE}(3)$, as illustrated in Fig.\ref{fig:results_of_layered_graph}. 
To ensure consistency, we fixed the neighbor radius and used the same random seed when generating the search graphs, which allows the densest layer of MRFMT$^*$ to match the search graph of FMT$^*$.
Our results demonstrate that MRFMT$^*$ achieves the same success rate as FMT$^*$ across various graph structures, which proves the solution completeness property of MRFMT$^*$. 
Notably, MRFMT$^*$ with linearly increasing layer density yields the lowest solution cost. 
Conversely, although MRFMT$^*$ with exponentially increasing layer density incurs a higher solution cost due to its sparse graph structure, it requires less time to find a feasible solution.
These findings suggest that the multi-resolution parameters controlling sample density across different layers should be carefully selected to balance solution cost and planning time based on task requirements. 
However, the determination of optimal parameters is beyond the scope of this paper.

\begin{figure*}[h]
\centering
  \subfloat{\includegraphics[width = 0.24\textwidth]{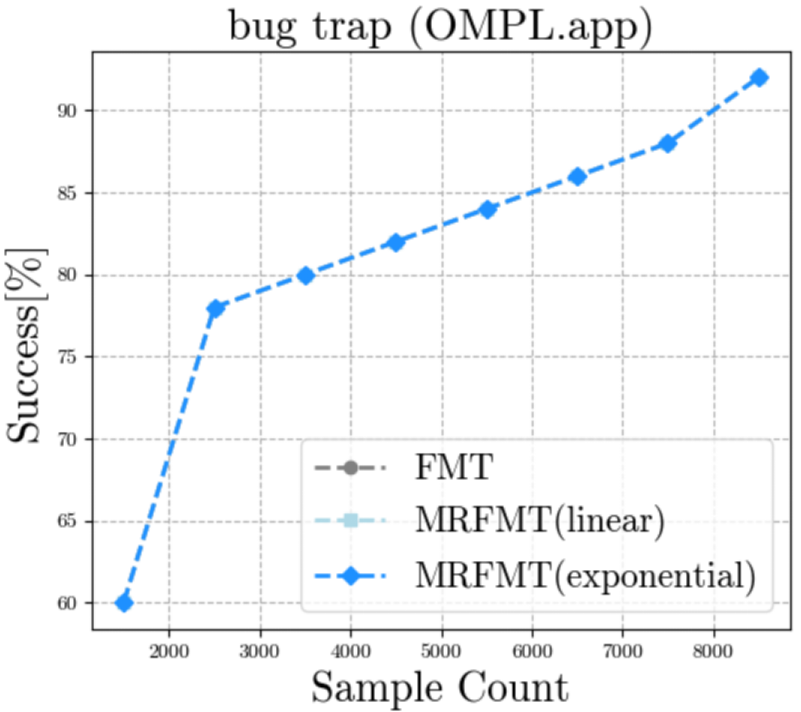}}\hfil
  \subfloat{\includegraphics[width = 0.24\textwidth]{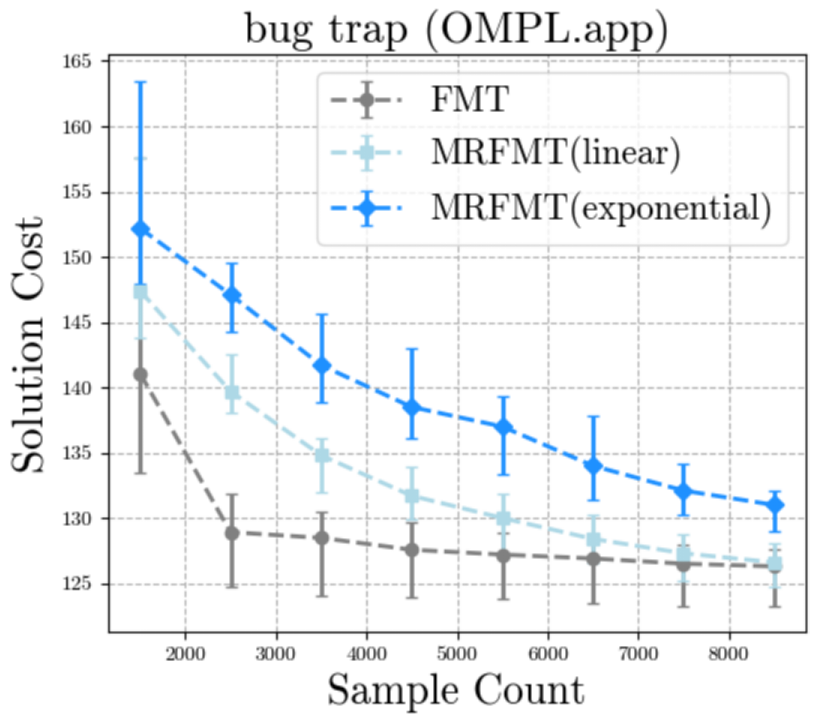}}\hfil
  \subfloat{\includegraphics[width = 0.24\textwidth]{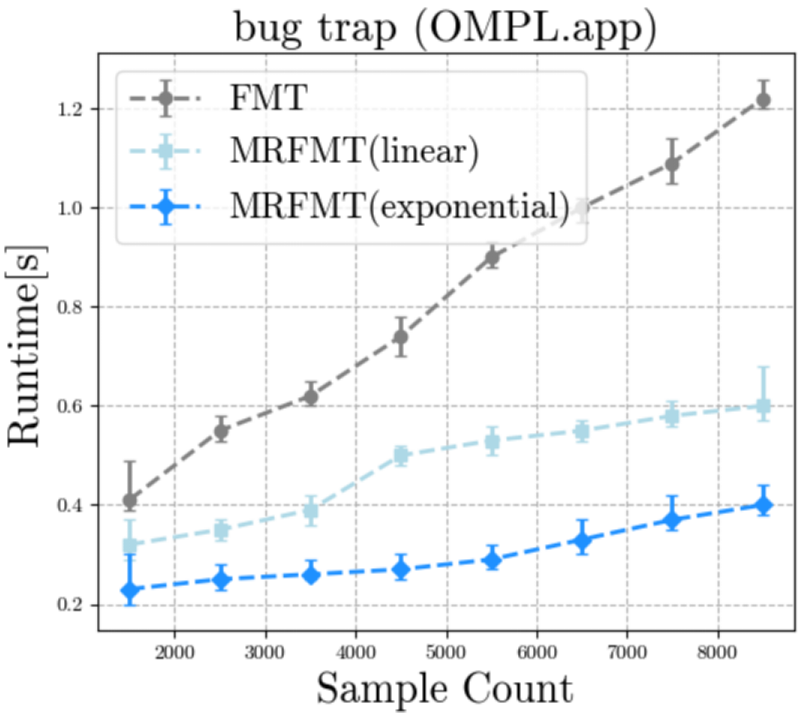}}
  
  \subfloat{\includegraphics[width = 0.24\textwidth]{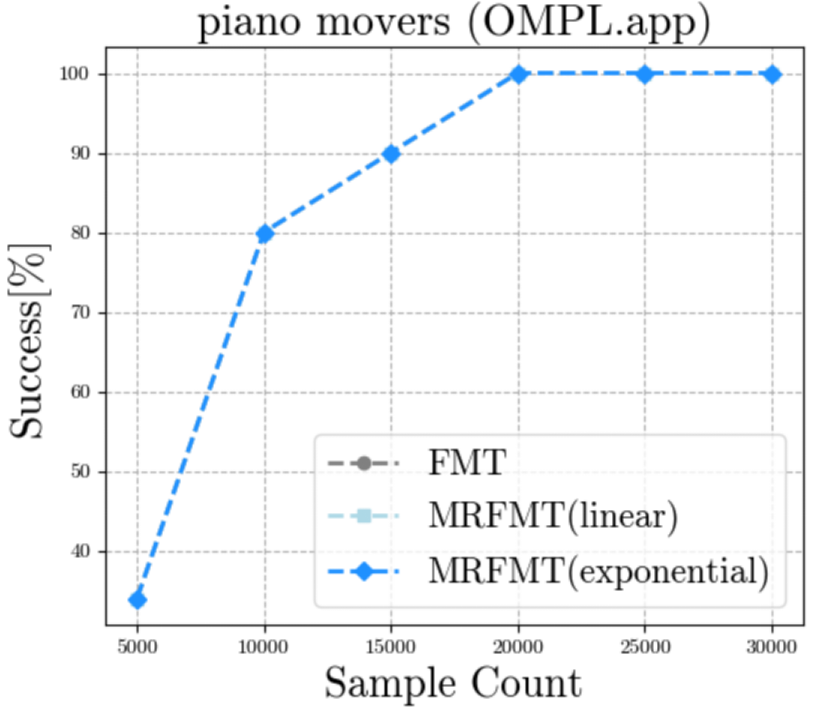}}\hfil
  \subfloat{\includegraphics[width = 0.24\textwidth]{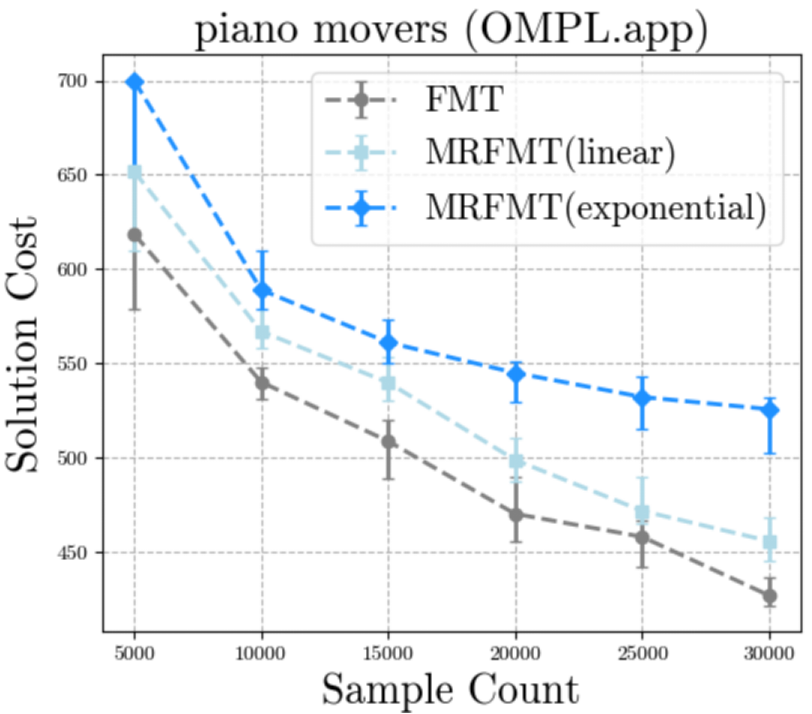}}\hfil
  \subfloat{\includegraphics[width = 0.24\textwidth]{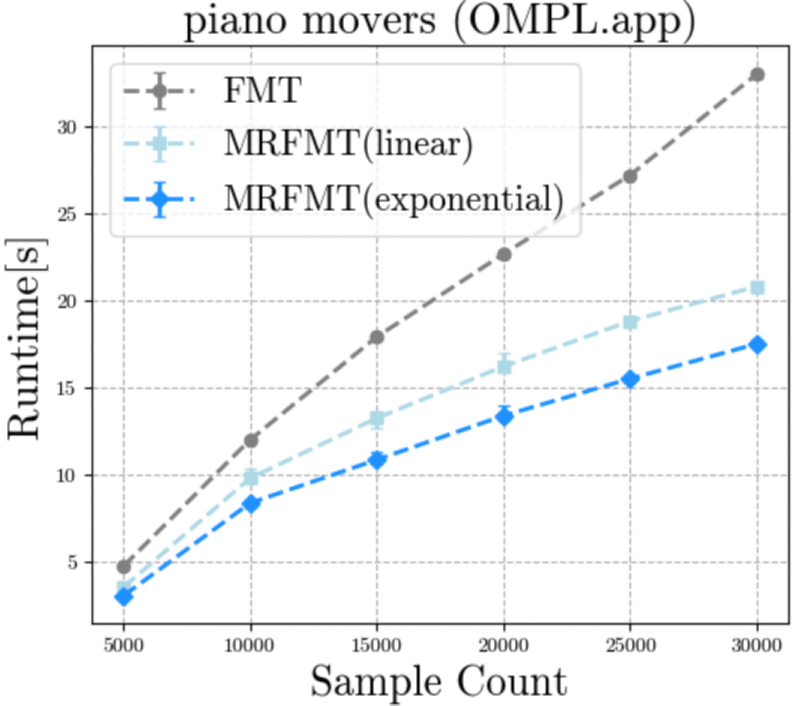}}
 
\caption{Planner performance versus sample count. Each planner was run 50 different times. 
The median values are plotted with error bars denoting a non-parametric 95$\%$ confidence interval on the median.
\textcolor{black}{The lines in the success rate vs. sample count plot overlap.}
}

\label{fig:results_of_layered_graph}
\end{figure*}

\section{Conclusion}
\label{Section:Conclusion}
We proposed MRFMT$^*$, an asymptotically optimal sampling-based planner with the selective densification strategy. 
MRFMT$^*$ is able to seamlessly transition between sparse and dense probabilistic approximations of configuration spaces, enabling it to achieve fast performance by searching over sparser approximations to navigate through large free state space and only densifying when tackling narrow passages.
The bidirectional version of MRFMT$^*$ further reduces search cost by simultaneously propagating search wavefront from two sources.
We present a theoretical analysis for MRFMT$^*$ regarding its completeness and optimality.
The simulation and experiment results show that MRFMT$^*$ and its bidirectional version can perform rapid online planning in high-dimensional state spaces with narrow passages by combining the advantages of planners with various granularities.
With their adaptability, the proposed planners can be easily implemented as a plug-and-play solution for diverse robotic systems, such as humanoid robots and
robotic arms, enabling them to perform tasks automatically and efficiently in unstructured environments.

\end{document}